\documentclass[letterpaper]{article} % DO NOT CHANGE THIS
\usepackage[preprint]{aaai2027}  % DO NOT CHANGE THIS
\usepackage[hyphens]{url}  % DO NOT CHANGE THIS
\usepackage{graphicx} % DO NOT CHANGE THIS
\usepackage{natbib}  % DO NOT CHANGE THIS AND DO NOT ADD ANY OPTIONS TO IT
\usepackage{caption} % DO NOT CHANGE THIS AND DO NOT ADD ANY OPTIONS TO IT
\usepackage{algorithm}
\usepackage{algorithmic}
\usepackage{graphicx}

\usepackage[most]{tcolorbox}
\usepackage{tikz}
\usetikzlibrary{arrows.meta,positioning,shapes.geometric,calc}

\usepackage{amsmath,amssymb,amsthm,mathtools,mathrsfs}
\usepackage{multirow}
\usepackage{amsfonts}
\usepackage{newfloat}
\usepackage{listings}
\DeclareCaptionStyle{ruled}{labelfont=normalfont,labelsep=colon,strut=off} % DO NOT CHANGE THIS
\floatstyle{ruled}
\newfloat{listing}{tb}{lst}{}
\floatname{listing}{Listing}

\usepackage{booktabs}

\theoremstyle{plain}
\newtheorem{theorem}{Theorem}[section]

\newtheorem{lemma}[theorem]{Lemma}

\newtheorem{corollary}[theorem]{Corollary}
\theoremstyle{definition}

\usepackage{booktabs}
\usepackage{multirow}

\title{
ExpBoN: Exponential-Noise Best-of-$n$ for Efficient Test-Time LLM Alignment
}
\author{
Yanxiao Liu$^{1,\star}$, Sicheng Wan$^{2,\star}$, and Deniz G\"und\"uz$^1$
}

\affiliations{
$^1$Imperial College London,
$^2$University of Washington, $^\star$equal contribution
}

\begin{document}

\maketitle

\begin{abstract}
Best-of-$n$ (BoN) sampling is a simple yet effective inference-time alignment method, but hard maximization provides only coarse control over the trade-off between reward and distribution shift.
Soft Best-of-$n$~\citep{verdun2025soft} provides smoother control and converges to the optimal distribution associated with KL-regularized reward maximization.
In this paper, we introduce ExpBoN, an alternative soft BoN method based on the exponential-noise report-noisy-max mechanism.
It admits an exact finite-$n$ decomposition, which yields exponentially fast convergence in total variation, expected reward, and both directions of KL divergence.
We provide comprehensive theoretical analyses of its convergence and regret behavior. 
We further integrate ExpBoN into the guided speculative inference (GSI) framework~\citep{geuter2025guided}, resulting in ExpGSI, for efficient reward-guided LLM alignment. 
ExpGSI yields substantial reductions in computational cost while maintaining comparable accuracy.
Experiments on MATH500, MMLU-STEM, and Minerva Math with the Qwen2.5-Math and Qwen3 model families show that ExpGSI reduces estimated computation by $14\%$-$39\%$ across candidate budgets for Qwen2.5-Math and by up to $45\%$ at $n=16$ for Qwen3. 
Overall, our results provide a theoretical and algorithmic foundation for exponential-noise BoN and efficient test-time LLM alignment.
\end{abstract}

% Uncomment the following to link to your code, datasets, an extended version or similar.
% You must keep this block between (not within) the abstract and the main body of the paper.
% Make sure that you do not de-anonymize yourself with these links.
% \begin{links}
%     \link{Code}{https://aaai.org/example/code}
%     \link{Datasets}{https://aaai.org/example/datasets}
%     \link{Extended version}{https://aaai.org/example/extended-version}
% \end{links}

\section{Introduction}

Large language models (LLMs) are usually trained to predict the next token rather than to directly optimize human preferences.  As a result, a pretrained model may generate responses that are inconsistent with the intended behavior~\citep{bender2021dangers, bommasani2021opportunities}.
\emph{Alignment} methods seek to modify the model distribution so that high-quality responses receive more probability while the aligned distribution remains close to the reference model.
Training-time and post-training alignment methods include Reinforcement Learning from Human Feedback (RLHF)~\citep{christiano2017deep, ouyang2022training}, SLiC~\citep{zhao2022calibrating}, and Direct Preference Optimization~\citep{rafailov2023direct}.  
In contrast, inference-time alignment is attractive because it can be applied without updating the base-model parameters and is often easier to deploy; examples include controlled decoding~\citep{mudgal2024controlled} and Best-of-$n$ (BoN) sampling~\citep{stiennon2020learning, beirami2024theoretical}.

BoN sampling is a simple inference-time alignment method that draws $n$ candidates from a reference distribution $P$ and returns the one with the largest reward.  
BoN requires no fine-tuning and often gives substantial reward improvements.  
Its performance has been shown to be empirically competitive with or superior to that of RLHF and other alignment schemes in terms of the \textit{reward-KL divergence trade-off}~\citep{mudgal2024controlled}, and it asymptotically approximates the solution of the KL-regularized reward optimization problem~\citep{yang2024asymptotics}. 
Its theoretical properties, coverage, and optimality have been studied extensively \citep{beirami2024theoretical, gui2024bonbon, mroueh2024information, amini2025variational, huang2025best, sriraman2026revisiting}, and it has been further applied to different scenarios~\citep{sun2024fast, raman2025adabon, jinnai2025regularized, kalayci2025optimal, ichihara2025evaluation, kang2026scalable, kobayashi2026flexible, mukherjee2026testtime, hsu2026best}.

Nevertheless, vanilla BoN, which is usually called the \emph{hard} BoN, has two limitations.
First, the number of candidates $n$ provides only coarse control: increasing $n$ simultaneously increases reward optimization and distribution shift.
Second, hard maximization can exploit errors in a learned proxy reward model.
As the candidate pool becomes larger, the selected response may have a high proxy reward but a lower true reward, a phenomenon commonly known as reward hacking or overoptimization~\citep{gao2023scaling, khalaf2026inference, bu2025consistency, aminian2025best}.

Soft Best-of-$n$ (SBoN) was introduced to provide finer control over this trade-off~\citep{verdun2025soft}.
Given samples $X_1,\ldots,X_n\stackrel{\mathrm{iid}}{\sim} P$, SBoN returns $X_i$ with probability proportional to $e^{r(X_i)/\lambda}$, and hence interpolates between sampling from $P$ and hard BoN.
For every temperature $\lambda$, the output distribution of SBoN converges to the \emph{optimal} target distribution~\citep{csiszar1975divergence} as $n$ grows.
It was further analyzed by~\citet{aminian2025best}, who studied its KL divergence from the reference policy and its true-reward regret under proxy-reward misspecification.
See also recent applications of SBoN on reward hacking~\citep{khalaf2026inference}, diffusion language models~\citep{bu2026dprm}, and guided speculative inference (GSI)~\citep{geuter2025guided}.

In this paper, we introduce an alternative soft BoN mechanism, called \textbf{ExpBoN}, that is more sample-efficient, converges to the optimal distribution geometrically, whereas the SBoN guarantees are polynomial. 
While SBoN can be shown to be equivalent to report-noisy-max with Gumbel noise, our ExpBoN instead relies on report-noisy-max with \textbf{exponential noise}.
The use of exponential noise is well known in differential privacy~\citep{mcsherry2007mechanism, dwork2014algorithmic}, where it admits desirable stability and utility properties, and has recently been used in LLM decoding to achieve Pareto optimality in the stability-perplexity trade-off \citep{zhao2024permute}. 
SBoN is closely related to the Poisson functional representation~\citep{li2018strong}, which is widely used in information theory~\citep{liu2025one, liu2025nonasymptotic, liu2024hiding}, and its connection to differential privacy has also been studied~\citep{liu2024universal, flamich2026scalable}. 
It was shown by~\citet{ding2021permute} to be equivalent to the permute-and-flip scheme~\citep{mckenna2020permute}.
Here, we bring this idea to test-time LLM alignment, provide comprehensive theoretic analyses, and show that it has desirable properties for BoN sampling.
In turn, it can be naturally integrated into the GSI framework~\citep{geuter2025guided}, making alignment and decoding even more efficient.

In summary, our main contributions are as follows.

\begin{itemize}
    \item We introduce ExpBoN, a BoN mechanism that utilizes exponential-noise report-noisy-max, a simple yet powerful idea with various desirable properties.
    
    \item We provide a comprehensive theoretical analysis of ExpBoN, including an exact decomposition property, exponentially fast convergence in both TV and KL divergences, and regret guarantees.
    
    \item To further accelerate LLM test-time alignment and decoding, we integrate ExpBoN into the GSI framework~\citep{geuter2025guided}, a natural combination that yields substantial reductions in computational cost while achieving comparable accuracy. 
\end{itemize}

\section{BoN with Exponential Noise}

Let a reference distribution $P$ over finite $\mathcal{X}$ denote the distribution of an LLM's responses to a prompt, and $r:\mathcal{X}\to\mathbb{R}$ denote a reward function that is nonconstant on $\mathcal X_P := \{x\in\mathcal X:P(x)>0\}$ and serves as a scoring guide to evaluate the alignment of LLM's responses with desired human values. 
The \emph{alignment} problem aims to find a distribution $P^\star$ that remains close to $P$ to preserve model quality, while maximizing the expected reward to align with human preferences: 
\begin{align*}
    \max_{P^\star \in \Delta(\mathcal X)} & \mathbb{E}_{P^\star} [r(X)] \\
     \text{subject to } & D_{\mathrm{KL}}(P^\star\Vert P)\leq \epsilon.
\end{align*}
This is equivalent to an information projection problem~\citep{csiszar2003information}, whose solution is known to be an exponentially tilted distribution~\citep{csiszar1975divergence} that assigns greater weight to high-reward responses by tilting $P$.
The optimal tilted distribution is 
\begin{equation*}
    P_\lambda^\star(x)
    :=
    \frac{P(x)e^{r(x)/\lambda}}{\mathbb E_P[e^{r(X)/\lambda}]},
\end{equation*}
where $\lambda>0$ is the temperature parameter that interpolates reward maximization and distribution divergence. 
See \citep{verdun2025soft} for more details.

However, in general we have neither  access to an analytical expression for $P$ nor to $r(\cdot)$, and therefore, we are unable to compute $P^\star$ directly. 
We can only sample from $P$ and evaluate $r(X)$ for these samples.
Accordingly, one simple yet effective test-time alignment scheme is BoN sampling~\citep{stiennon2020learning}, which draws $n$ samples from $P$ and selects the one with the highest reward.
The BoN consists of: 
\begin{enumerate}
    \item Draw $X_1,\ldots, X_n$ i.i.d. from $P$; 

    \item Compute rewards $r(X_1), \ldots, r(X_n)$; 

    \item Return $Y=X_K$ where $K = \mathrm{arg} \max_{k=1,\ldots,n} r(X_k) $.
\end{enumerate}

In~\citep{verdun2025soft}, a more general framework, called soft BoN (SBoN), is proposed to provide a smoother approach that interpolates between reward maximization and distribution divergence. 
SBoN consists of: 

\begin{enumerate}
    \item Draw $X_1,\ldots, X_n$ i.i.d. from $P$; 

    \item Compute rewards $r(X_1),\ldots,r(X_n)$;

    \item Draw $K$ from $\{1, \ldots, n\}$ with respect to 
    \begin{equation}
        \label{eq::SBoN_dist}
        \Pr(K=i) = \frac{e^{r(X_i)/\lambda}}{\sum^n_{j=1} e^{r(X_j)/\lambda}},
    \end{equation}
    and return $Y=X_K$. 
\end{enumerate}

We can see that SBoN recovers BoN as $\lambda\to 0$. 
It can be shown that (Appendix~\ref{subapp::SBoN_gumbel}) the last step above, which samples according to~\eqref{eq::SBoN_dist}, is equivalent to report-noisy-max with Gumbel noise~\citep{Gumbel1954statistical, huijben2022review, li2018strong}, i.e., taking $Y=X_K$, where
\begin{equation*}
K = \arg\max_{k=1,\ldots,n} \left\{ \frac{r(X_k)}{\lambda} + G_k \right\},
\end{equation*}
and $G_k\sim \mathrm{Gumbel}(0,1)$ are independent of each other.

We introduce an alternative soft version of BoN sampling that utilizes exponential noise rather than Gumbel noise: 
\begin{tcolorbox}[
title={Exponential-Noise Best-of-$n$ (ExpBoN)},
colback=white,
colframe=black,
colbacktitle=white,
coltitle=black,
fonttitle=\bfseries,
titlerule=0.8pt,
boxrule=0.8pt,
arc=0pt,
left=6pt,
right=6pt,
top=6pt,
bottom=6pt,
toptitle=1pt,
bottomtitle=1pt
]
\begin{enumerate}
\item Draw $X_1,\ldots,X_n$ i.i.d. from $P$;

\item Compute rewards $r(X_1),\ldots,r(X_n)$;

\item Draw $E_1,\ldots,E_n$ i.i.d. from $\mathrm{Exp}(1)$, independently;

\item Return $Y=X_K$, where
\begin{equation*}
    K = \arg\max_{k=1,\ldots,n}
    \left\{ \frac{r(X_k)}{\lambda} + E_k \right\}.
\end{equation*}

\end{enumerate}
\end{tcolorbox}

The scheme has the same candidate-generation and reward-evaluation procedure as BoN and SBoN, and the same limiting behavior: when $n=1$ or $\lambda\to\infty$, the output follows $P$; when $\lambda\downarrow0$, it approaches hard BoN; and, for every fixed $\lambda>0$, it converges to the tilted target $P_\lambda^\star$ as $n\to\infty$.  
The difference is that utilizing the exponential noise creates an exact tilted component at finite $n$, which in turn provides advantages in both theoretic foundation and experiments.

\section{Theoretic Guarantees}
\label{sec:analysis}

We provide comprehensive theoretical analyses of ExpBoN, and compare to BoN and SBoN~\citep{verdun2025soft} baselines. 
We derive an exact finite-$n$ decomposition representation of ExpBoN, which shows an exponential convergence to the optimal distribution, rather than $O(1/n)$ rate by SBoN. 
% ; in contrast, the SBoN mechanism only has $O(1/n)$ reverse-KL and relative-reward guarantees. 

\subsection{Exact Decomposition and Convergence}
\label{subsec::decomp_converge}

For $k = 1,\ldots,n$ and $E_k \sim \mathrm{Exp}(1)$, we denote 
\begin{equation*}
    M_n:= \max_{k=1,\ldots,n}T_k \quad \text{ where } \quad  T_k
    :={r(X_k)}/{\lambda} + E_k. 
\end{equation*}

The following exact decomposition of the ExpBoN distribution serves as the basis for most of the subsequent results.

\begin{theorem}
\label{thm:ExpBoN_decomposition}
For every $n\geq 1, \lambda>0$, there exists a distribution
$Q_{n,\lambda}$ on $\mathcal X_P$ such that the ExpBoN output distribution
\begin{equation*}
    \widetilde P_{n,\lambda}
    =
    \left(
        1-\rho_\lambda^n
    \right)
    P_\lambda^\star
    +
    \rho_\lambda^n Q_{n,\lambda}, 
\end{equation*}
where $\rho_\lambda := 1- \mathbb E_P\left[e^{(r(X)-r_{\max})/\lambda}\right], \, r_{\max}:=\max_x r(x)$. 
\end{theorem}

In comparison, SBoN only admits finite-sample approximation bounds~\citep[Lemma 1]{verdun2025soft}. 
We call
\begin{equation*}
\alpha_{n,\lambda}
:=
1-\rho_\lambda^n
\end{equation*}
as the \emph{tilted-hit probability}.  It is the probability of the coupling event ${M_n\geq s_{\max}}$, conditional on which the ExpBoN output law is exactly $P_\lambda^\star$. 
For a fixed temperature $\lambda$, increasing $n$ increases the weight of the exact tilted component.

The exact decomposition leads to tight TV and KL guarantees for ExpBoN. Unlike~\citet{verdun2025soft}, we provide results for both the forward and reverse KL divergences. 
\begin{theorem}
\label{thm:pf_sbon_kl_tv}
For every $n\geq 1$ and $\lambda>0$, assuming $r$ is not constant on $\mathcal X_P$, we have 
\begin{align*}
    q_\lambda\rho_\lambda^n  \leq \mathrm{TV}
    \big(
        P_\lambda^\star,
        \widetilde P_{n,\lambda}
    \big)
    & \leq
    \rho_\lambda^n,  \\ 
2q_\lambda^2\rho_\lambda^{2n}
\leq D_{\mathrm{KL}}
\big(
    \widetilde P_{n,\lambda}
    \Vert
    P_\lambda^\star
\big)
& \leq
\log\left(
    1+
    \left(p_{\min,\lambda}^{-1} - 1\right)
    \rho_\lambda^{2n}
\right), \\
2q_\lambda^2\rho_\lambda^{2n}
 \leq D_{\mathrm{KL}}
\big(
    P_\lambda^\star
    \Vert
    \widetilde P_{n,\lambda}
\big) 
& \leq
\min\Big\{
    -\log \alpha_{n,\lambda}, 
    \\
    & \qquad
    { \left(p_{\min,\lambda}^{-1}-1\right)
        \rho_\lambda^{2n}
    }\big /{\alpha_{n,\lambda}}
\Big\}, 
\end{align*}
where $q_\lambda := P_\lambda^\star(\{x:r(x)=r_{\max}\})$ and $p_{\min,\lambda} := \min_{x\in\mathcal X_P}P_\lambda^\star(x)$; and therefore, we have, as $n \to \infty$, 
\begin{align*}
    \mathrm{TV}
    \big(
        P_\lambda^\star,
        \widetilde P_{n,\lambda}
    \big)
    &=
    \Theta(\rho_\lambda^n), \\
    D_{\mathrm{KL}}
    \big(
        P_\lambda^\star
        \Vert
        \widetilde P_{n,\lambda}
    \big)
    &=
    \Theta(\rho_\lambda^{2n}), \\
    D_{\mathrm{KL}}
    \big(
        \widetilde P_{n,\lambda}
        \Vert
        P_\lambda^\star
    \big)
    &=
    \Theta(\rho_\lambda^{2n}).
\end{align*}
\end{theorem}

Compared to SBoN, which only has $D_{\mathrm{KL}}(P_\lambda^\star\Vert P^{\mathrm S}_{n,\lambda})=O(1/n)$ \citep{verdun2025soft}, for every finite $(P,r,\lambda)$, ExpBoN improves these polynomial convergence rates.

\begin{figure*}[t]
% \begin{figure}[htpb]
% \begin{figure}[H]
    \centering 
    \includegraphics[scale = 0.39]{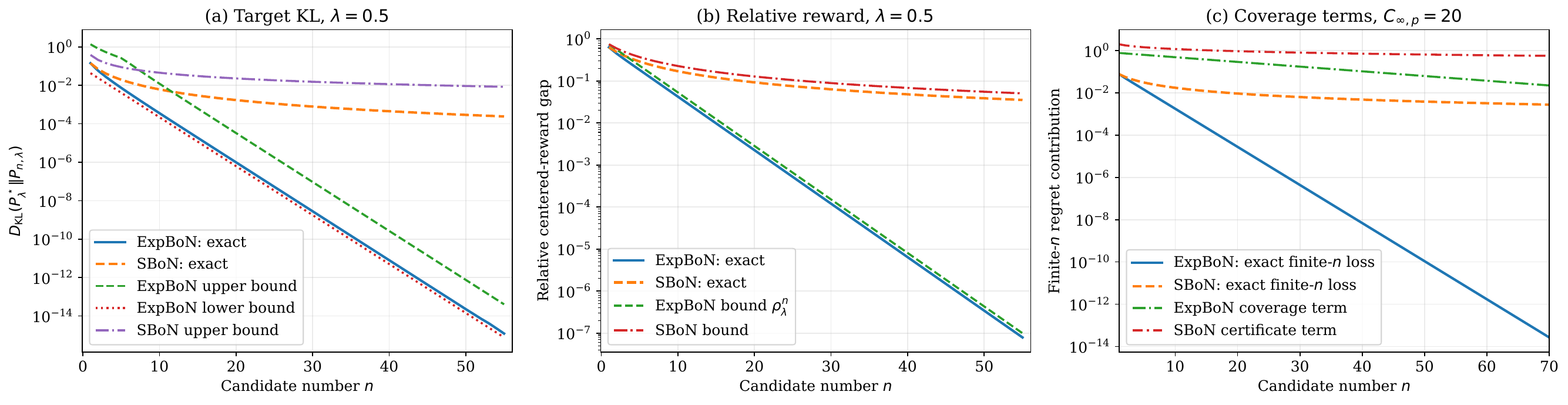} 
    \caption{
    Finite-$n$ comparison of ExpBoN and SBoN on the example $P=(0.75,0.20,0.05)$.  
    Panels~(a) and (b) use $r=(0.016,0.164,0.820)$ and $\lambda=0.5$. 
    Panel~(a) compares the KL divergences with the ExpBoN  bounds in Theorem~\ref{thm:pf_sbon_kl_tv} and the $O(1/n)$ SBoN bound. 
    Panel~(b) compares the relative centered-reward gaps with the ExpBoN bound $\rho_\lambda^n$ from Corollary~\ref{cor:ExpBoN_relative_reward} and the corresponding $O(1/n)$ SBoN bound. 
    Panel~(c) uses $r$ as the true reward and $r_p=r+0.1(1,0,-1)$ as the proxy reward at $\lambda=0.5$, and compares the exact finite-$n$ true-reward losses from the common proxy-tilted limit with the ExpBoN bound in Theorem~\ref{thm:ExpBoN_coverage_regret} and the SBoN bound. 
    The results show the exponential rates of ExpBoN in contrast to SBoN's polynomial bounds.
    }
    \label{fig:convergence} 
\end{figure*}

\subsection{Reward and KL Convergence}
\label{subsec:ExpBoN_reward_convergence}

We then study how finite-$n$ ExpBoN approaches the tilted target in expected reward and in the KL-regularized objective. 
We provide a one-sided reward guarantee and show that the expected reward improves monotonically with samples. 

\begin{theorem}[Reward convergence]
\label{thm:ExpBoN_reward_objective}
For $n\geq1$ and $\lambda>0$,
\begin{equation*}
    0
    \leq
    \mathbb E_{P_\lambda^\star}[r(X)]
    -
    \mathbb E_{\widetilde P_{n,\lambda}}[r(X)]
    \leq
    \Delta_r\rho_\lambda^n
\end{equation*}
where $\Delta_r:=r_{\max}-r_{\min}$. 
For every $\lambda > 0$, $\mathbb E_{\widetilde P_{n,\lambda}}[r(X)]$ is nondecreasing in $n$, and converges to
$\mathbb E_{P_\lambda^\star}[r(X)]$ as $n\to\infty$. 
\end{theorem}

Define the centered reward $ \overline r(x):=r(x)-r_{\min}$, we have:
\begin{corollary}
\label{cor:ExpBoN_relative_reward}
For every $n\geq1$ and $\lambda>0$,
\begin{equation*}
    0
    \leq
    \frac{
        \mathbb E_{P_\lambda^\star}[\overline r(X)]
        -
        \mathbb E_{\widetilde P_{n,\lambda}}[\overline r(X)]
    }{
        \mathbb E_{P_\lambda^\star}[\overline r(X)]
    }
    \leq
    \rho_\lambda^n.
\end{equation*}
\end{corollary}

% Define the KL-regularized objective
% \begin{equation*}
%     \mathcal F_\lambda(Q)
%     :=
%     \mathbb E_Q[r(X)]
%     -
%     \lambda D_{\mathrm{KL}}(Q\Vert P).
% \end{equation*}

% \begin{corollary}[Regularized-objective convergence]
% \label{cor:ExpBoN_regularized_objective}
% For every $n\geq1$ and $\lambda>0$,
% \begin{align*}
% &\mathcal F_\lambda(P_\lambda^\star)
% -
% \mathcal F_\lambda(\widetilde P_{n,\lambda})
% \\
% &\qquad=
% \lambda
% D_{\mathrm{KL}}
% \big(
%     \widetilde P_{n,\lambda}
%     \Vert
%     P_\lambda^\star
% \big)
% \\
% &\qquad\leq
% \lambda
% \log\left(
%     1+
%     \left(p_{\min,\lambda}^{-1}-1\right)
%     \rho_\lambda^{2n}
% \right).
% \end{align*}
% \end{corollary}

Compared to \citet[Theorem~3]{verdun2025soft}, which provided an $O(1/n)$ relative-reward guarantee for SBoN, Corollary~\ref{cor:ExpBoN_relative_reward} gives an exponential $O(\rho_\lambda^n)$ guarantee.

The following result shows that the policy-to-reference KL of finite-$n$ ExpBoN approaches that of the ideal tilted policy.

\begin{corollary}
\label{cor:ExpBoN_reference_kl_localization}
For every $n\geq1$ and $\lambda>0$,
\begin{align*}
&D_{\mathrm{KL}}
\big(
P_\lambda^\star\Vert P
\big) - 
\frac{\Delta_r}{\lambda}\rho_\lambda^n \leq
D_{\mathrm{KL}}
\big(
\widetilde P_{n,\lambda}\Vert P
\big)
\\
&\leq
D_{\mathrm{KL}}
\big(
P_\lambda^\star\Vert P
\big)
+
\log\big(
1+
(p_{\min,\lambda}^{-1}-1) 
\rho_\lambda^{2n}
\big).
\end{align*}
\end{corollary}
This shows that ExpBoN approaches the reward and reference-KL coordinates of the optimal tilted policy exponentially fast; we can also derive a distribution-free form 
\begin{equation*}
D_{\mathrm{KL}}
\big(
\widetilde P_{n,\lambda}\Vert P
\big)
\leq
\min\left\{
\log n,
{\Delta_r^2}/(2\lambda^2)
\right\},
\end{equation*}
which has the correct quadratic order near the reference regime, whereas \citet[Lemma~4.1]{aminian2025best} is linear in the inverse temperature.

\subsection{Results via Permute-and-Flip}

As discussed previously, report-noisy-max with exponential noise is equivalent to permute-and-flip~\citep{mckenna2020permute, ding2021permute}, which enjoys desirable properties.
Here, we utilize them to establish the following result.

\begin{theorem}[Utility Dominance]
\label{thm:ExpBoN_utility_dominance}
For every $n\geq1$, $\lambda>0$, and candidates $x_{1:n}=(x_1,\ldots,x_n)$, let $K_{\mathrm E}$ and $K_{\mathrm S}$
denote the indices selected by ExpBoN and SBoN, respectively.  Then
\begin{equation*}
    \mathbb E
    \left[
        r(X_{K_{\mathrm E}})
        \,\middle|\,
        X_{1:n}=x_{1:n}
    \right]
    \geq
    \mathbb E
    \left[
        r(X_{K_{\mathrm S}})
        \,\middle|\,
        X_{1:n}=x_{1:n}
    \right].
\end{equation*}
Consequently,
\begin{equation*}
    \mathbb E_{\widetilde P_{n,\lambda}}[r(X)]
    \geq
    \mathbb E_{P^{\mathrm{SBoN}}_{n,\lambda}}[r(X)].
\end{equation*}
\end{theorem}

ExpBoN inherits the \emph{Pareto-optimality} of permute-and-flip~\citep{mckenna2020permute}: 
within the equally stable candidate-level schemes, no one can uniformly improve the expected selection reward of ExpBoN over all reward vectors.

\subsection{Regret Analysis}
\label{subsec:regret_analysis}

While the preceding results characterize how accurately finite-$n$ ExpBoN approaches a prescribed distribution, this target may be defined using an imperfect proxy reward.  
Faster convergence does not by itself guarantee better true-reward performance and may amplify reward over-optimization.  
We therefore study the true-reward regret of ExpBoN, in parallel with the regret analysis of SBoN by~\citet{aminian2025best}.

Suppose the proxy reward and the true reward are $r_p$ and $r_t$, respectively. 
For $u\in\{p,t\}$, let  $\widetilde P_{n,\lambda}^{u}$ be the ExpBoN output distribution when $r_u$ is used, and define
\begin{equation*}
    P_{\lambda,u}^\star(x)
    :=
    \frac{
        P(x)e^{r_u(x)/\lambda}
    }{
        \mathbb E_P[e^{r_u(X)/\lambda}]
    }. 
\end{equation*}
Let $d(x):=r_t(x)-r_p(x)$ and define the centered reward-error quantity $\omega_d :=\max_{x\in\mathcal X_P}d(x) - \min_{x\in\mathcal X_P}d(x)$. 

\begin{theorem}[Proxy-Reward Stability]
\label{thm:ExpBoN_proxy_stability}
The ExpBoN policies induced by the true and proxy rewards satisfy
\begin{align*}
&\max\Big\{
D_{\mathrm{KL}}
\big(
    \widetilde P_{n,\lambda}^{t}
    \Vert
    \widetilde P_{n,\lambda}^{p}
\big),
D_{\mathrm{KL}}
\big(
    \widetilde P_{n,\lambda}^{p}
    \Vert
    \widetilde P_{n,\lambda}^{t}
\big)
\Big\} \\
& \qquad  \quad \leq
\min\Big\{
    \frac{(n-1)\operatorname{Var}_P(d(X))}{\lambda^2},
    \frac{\omega_d}{\lambda}
    \tanh\left(
        \frac{\omega_d}{2\lambda}
    \right)
\Big\}, 
\end{align*}
and also the TV guarantee
\begin{align*}
& \mathrm{TV}
\big(
    \widetilde P_{n,\lambda}^{t},
    \widetilde P_{n,\lambda}^{p}
\big)  \\
& \qquad \leq
\min \Bigg\{
    \sqrt{
        \frac{(n-1)\operatorname{Var}_P(d(X))}{2\lambda^2}
    },
    \tanh\left(
        \frac{\omega_d}{2\lambda}
    \right),
    1
\Bigg\}.
\end{align*}
\end{theorem}
In comparison, \citet[Lemma~4.2]{aminian2025best} analyzed the true-versus-proxy SBoN policy using an uncentered exponential mean-square error. 
Theorem~\ref{thm:ExpBoN_proxy_stability} can be  sharper for small centered errors, though there is no uniform dominance.

We then provide guarantees for the regret of ExpBoN. We first define $\Delta_t := \max_{x\in\mathcal X_P}r_t(x) - \min_{x\in\mathcal X_P}r_t(x)$ and 
\begin{equation*}
    \operatorname{Reg}_{n,\lambda}
    :=
    r_{t,\max}
    -
    \mathbb E_{\widetilde P_{n,\lambda}^{p}}[r_t(X)].
\end{equation*}

\begin{theorem}[Direct Regret Bound]
\label{thm:ExpBoN_direct_regret}
For every $n\geq1$ and $\lambda>0$,
\begin{equation*}
    \operatorname{Reg}_{n,\lambda}
    \leq
    B_t(\lambda)
    +
    \Delta_t
    \left[
        \tanh\left(
            \frac{\omega_d}{4\lambda}
        \right)
        +
        \rho_{\lambda,p}^{\,n}
    \right],
\end{equation*}
where $B_t(\lambda) := r_{t,\max} - \mathbb E_{P_{\lambda,t}^\star}[r_t(X)]$ is the softening bias of the true tilted policy and 
$\rho_{\lambda,u} := 1- \mathbb E_P \left[ e^{(r_u(X)-r_{u,\max})/\lambda} \right]$ is the tilted-hit failure probability associated with reward $r_u$.
\end{theorem}
Theorem~\ref{thm:ExpBoN_direct_regret} separates the total regret into softening bias, proxy-target mismatch, and a finite-$n$ target-approximation term.  
The last term is advantageous when the proxy-tilted target is trustworthy; under severe proxy misspecification, faster target approximation need not imply better true reward.

\citet{aminian2025best} studied the special scenario where $0\leq r_t(x),r_p(x)\leq R_{\max}$ for all $x\in\mathcal X_P$. 
Here we define: 
\begin{equation*}
    \overline\varepsilon_\lambda
    :=
    \inf_{c\in\mathbb R}
    \lambda
    \log
    \mathbb E_P
    \left[
        \exp\left(
            \frac{(d(X)-c)^2}{\lambda}
        \right)
    \right],
\end{equation*}
which measures the proxy-reward error, and we have: 

\begin{theorem}[Coverage-Based Regret Bound]
\label{thm:ExpBoN_coverage_regret}
For every $n\geq1$ and $\lambda>0$,
\begin{align*}
\operatorname{Reg}_{n,\lambda}
\leq{}&
\sqrt{\overline\varepsilon_\lambda}
\left(
    \sqrt{C_{\infty,t}}
    +
    \sqrt{C_{\infty,p}}
\right)
+
\lambda\log C_{\infty,t}
\\
&+
R_{\max}
\left(
    1-\frac{1}{C_{\infty,p}}
\right)^n.
\end{align*}
where $C_{\infty,u}:=P(\left\{ x\in\mathcal X_P: r_u(x)=r_{u,\max} \right\})^{-1}$. 
\end{theorem}

\citet[Theorem~5.2]{aminian2025best} gives the same reward-model-error and regularization-bias structure, but its displayed finite-$n$ term is $O(n^{-1/2})$ for fixed coverage, whereas the corresponding ExpBoN term is $R_{\max}\left(1-\frac{1}{C_{\infty,p}}\right)^n$. 
Therefore, making the finite-$n$ term at most $\eta$ needs $n = O\left(C_{\infty,p}\log\frac{R_{\max}}{\eta}\right)$ for ExpBoN, compared with $n=O\left(\frac{R_{\max}^2(C_{\infty,p}-1)}{\eta^2}\right)$ for the SBoN~\citep{aminian2025best}.

\paragraph{Numerical Example}

We use the example same in \citet[Figure~1]{verdun2025soft} to verify our theorems. 
Consider 
\begin{equation*}
    P=(0.75,0.20,0.05),
    \quad
    r=(0.016,0.164,0.820).
\end{equation*}
At $\lambda=0.5$, we obtain the parameter $\rho_\lambda=0.745928$ and $P_\lambda^\star=(0.591233,0.211972,0.196795)$. 
For every $n$, we compute and compare the finite-$n$ SBoN and ExpBoN marginal distributions. 
For example, at $n=10$, we have
\begin{align*}
&D_{\mathrm{KL}}
\big(
    P_\lambda^\star
    \Vert
    \widetilde P_{10,\lambda}
\big)
=3.58\times10^{-4}, \\
& D_{\mathrm{KL}}
\big(
    P_\lambda^\star
    \Vert
    P^{\mathrm{SBoN}}_{10,\lambda}
\big) =6.16\times10^{-3}, 
\end{align*}
and the relative centered reward gaps are $0.0423$ and $0.1706$, respectively. 
The theoretical results are compared in Figure~\ref{fig:convergence}.
The same separation holds on empirical candidate pools generated by the
draft model in the experiments of Section \ref{sec:experiments}, both for reward-only
scores and for the clipped score used by ExpGSI
(Appendix~\ref{sec:Additional experimental results},
Figure~\ref{fig:realconv}).

\begin{figure*}[t]
\centering
\includegraphics[width=0.9\textwidth]{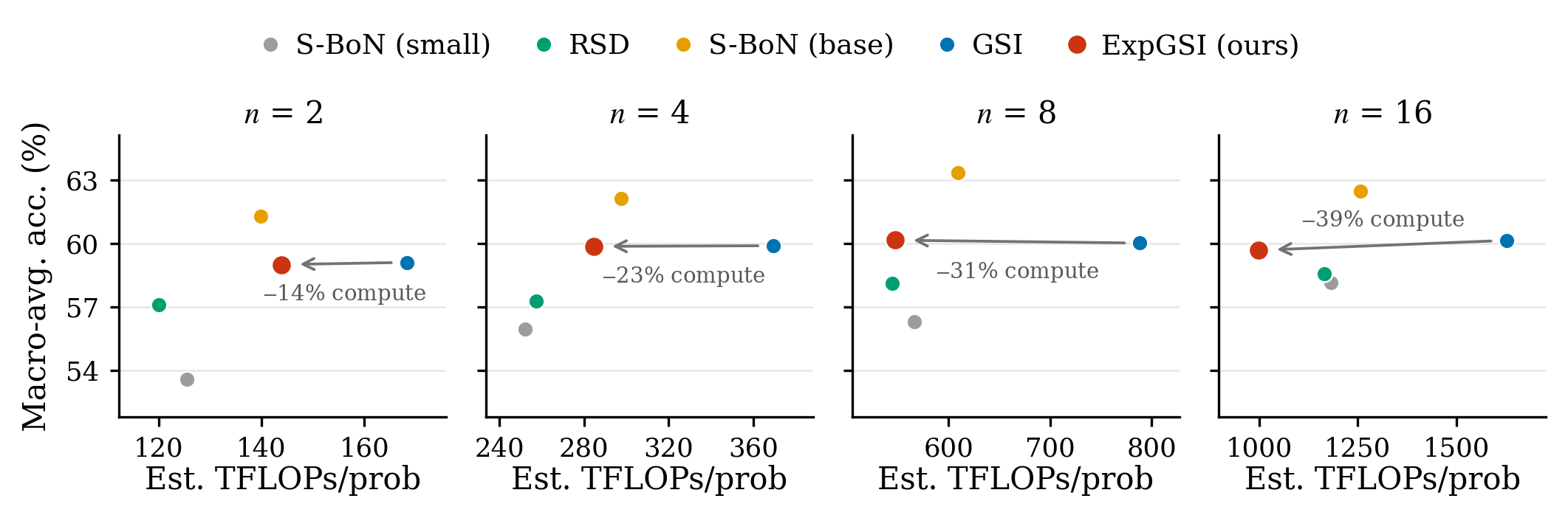}
\caption{Compute--accuracy planes at each candidate budget $n$ (Qwen2.5-Math; macro average over MATH500, MMLU-STEM, and Minerva Math; $3$-seed means). Arrows mark the GSI$\,\to\,$ExpGSI reduction in estimated compute. ExpGSI maintains accuracy comparable to GSI at every $n$ while its compute saving widens monotonically with $n$, from $14\%$ at $n{=}2$ to $39\%$ at $n{=}16$; full metrics for both model families are reported in Table~\ref{tab:model_family_average}.}

\label{fig:latency}
\end{figure*}

\section{ExpBoN with GSI}
\label{subsec:clipped_exp_gsi}

GSI~\citep{geuter2025guided} accelerates test-time alignment with SBoN by using a small draft model $\pi_S$ to propose candidate reasoning steps and a large base model $\pi_B$, together with a likelihood-ratio-corrected reward, to verify and select among them, thereby reducing the cost of generating every candidate autoregressively with $\pi_B$.

In this section, we show that ExpBoN, which admits an exact finite-$n$ decomposition and the resulting desirable properties, is a natural replacement for the SBoN component in the GSI pipeline.
It preserves the same optimal tilted target distribution while offering geometrically faster finite-candidate convergence and an exact, distribution-preserving early-stopping implementation based on truncated rejection sampling.
Our experiments demonstrate substantial computational savings without sacrificing answer accuracy.

We first briefly review the GSI pipeline~\citep{geuter2025guided}. 
At each reasoning state $h$, with draft model $\pi_S(\cdot\mid h)$ and base model $\pi_B(\cdot\mid h)$, the original GSI applies SBoN to a modified score $r(h,y) + d(h,y) / \beta$ where $d(h,y) := \log\frac{\pi_B(y \mid h)}{\pi_S( y \mid h)}$ and $\beta>0$ is a parameter that trades off maximizing the reward versus fidelity to $\pi_B$. 

We then explain our strategy. 
We replace the SBoN component by ExpBoN and consider a \emph{clipped} reward: 
\begin{equation*}
s_C(h,y)
:=
\beta r(h,y)
+
\min\{d(h,y), C\},
\end{equation*}
with a fixed $C > 0$. 
The clipping trick solves the issue that the likelihood-ratio $d(h,y)$ may not have a finite upper bound, and has also been used in \citep{sriraman2026revisiting} for analyzing BoN schemes. 
If $r(h,y)\leq R$, then $s_C(h,y)\leq U_C:=\beta R+C$ provides an envelope for the ExpBoN.

Clipping generally changes the original limiting tilted target, but its effect can be separated from the finite-candidate approximation error and provably bounded by Theorem~\ref{thm:clipped_exp_gsi_bias}, as will be elaborated later. 
We note that the clipping technique can also be employed in the original GSI framework with SBoN; what makes it special here is its natural connection to ExpBoN, whose exact decomposition yields both geometric finite-candidate approximation and a distribution-preserving early-exit implementation, as elaborated below.

Denoting the target tilted distribution by $\pi_{\beta,C}^{\star}(y\mid h)$, 
the ExpBoN exact decomposition property to the clipped score $s_C(h,y)$ gives an output distribution $\widetilde\pi_{n,\beta,C}(\cdot\mid h)$ as 
\begin{equation*} 
\left(
1-\rho_{\beta,C}(h)^n
\right)
\pi_{\beta,C}^{\star}(\cdot\mid h)
+
\rho_{\beta,C}(h)^n
Q_{n,\beta,C}(\cdot\mid h)
\end{equation*}
where $\rho_{\beta,C}(h) := 1- \mathbb E_{\pi_S} \left[ e^{s_C(h,Y)-U_C} \right]$ and $Q_{n,\beta,C}(\cdot\mid h)$ is some residual distribution.
The hit branch has a \emph{rejection-sampling} representation: draw a draft candidate $Y\sim\pi_S(\cdot\mid h)$ and accept it with probability $e^{s_C(h,Y)-U_C}$; conditional on acceptance, $Y$ is an \emph{exact} sample from $\pi_{\beta,C}^{\star}(\cdot\mid h)$. 
Moreover, by exponential memorylessness, conditional on crossing the threshold, the overshoot $s_C(h,Y)+E-U_C$ is $\operatorname{Exp}(1)$ and is independent of $Y$; therefore, selecting the largest overshoot, as full ExpBoN does, and selecting the first accepted proposal in an independent random order induce the same hit-branch law. 
Thus, clipped ExpBoN can be implemented as a truncated rejection sampler: return the first accepted proposal in an independent random order if any of the $n$ proposals is accepted, and otherwise evaluate all candidates and return the exponential-noise maximizer. 
Overall, it is provably an exact finite-$n$ clipped ExpBoN sampler. 
By contrast, under Gumbel noise the analogous fixed-threshold crossing probability $1-\exp\left(-e^{s_i-U}\right)$ is not proportional to $e^{s_i}$, so the same first-crossing construction would not exactly preserve the SBoN sampling \citep{verdun2025soft}.

\begin{algorithm}[tb]
\caption{ExpBoN-based GSI}
\label{alg:clipped_exp_gsi}
\textbf{Input}: State $h$; draft policy $\pi_S$; base policy $\pi_B$;
reward model $r$; candidate count $n$; first-batch size $b$;
inverse temperature $\beta$; clipping level $C$; reward upper bound $R$;
GSI threshold $u$\\
\textbf{Output}: Selected reasoning step
\begin{algorithmic}[1]
\STATE Sample
$y_1,\ldots,y_n\overset{\mathrm{iid}}{\sim}\pi_S(\cdot\mid h)$
and record
$\ell_{S,i}\leftarrow\log\pi_S(y_i\mid h)$
\STATE Sample an independent random permutation
$\sigma$ of $\{1,\ldots,n\}$ and independent
$E_i\sim\operatorname{Exp}(1)$
\STATE $U_C\leftarrow\beta R+C$
\STATE Form batches $\mathcal{B}_1 \gets$ the first $b$ indices of the
permuted order, $\mathcal{B}_2 \gets$ the remaining $n-b$
\FOR{$j=1,2$}
    \STATE Evaluate, in parallel for $i\in\mathcal B_j$,
    \[
        r_i\leftarrow r(h,y_i),
        \qquad
        \ell_{B,i}\leftarrow\log\pi_B(y_i\mid h)
    \]
    \STATE
    $d_i\leftarrow\ell_{B,i}-\ell_{S,i}$ and
    $s_{C,i}\leftarrow\beta r_i+\min\{d_i,C\}$ 
    \STATE
    $\mathcal H_j
    \leftarrow
    \{i\in\mathcal B_j:s_{C,i}+E_i\geq U_C\}$
    \IF{$\mathcal H_j\neq\varnothing$}
        \STATE
        $i^\star
        \leftarrow
        \arg\min_{i\in\mathcal H_j}\sigma^{-1}(i)$
        \STATE \textbf{go to} Line~\ref{line:clipped_exp_gsi_gate}
    \ENDIF
\ENDFOR
\STATE
$i^\star
\leftarrow
\arg\max_{i=1,\ldots,n}\{s_{C,i}+E_i\}$
\STATE \label{line:clipped_exp_gsi_gate}
$\widetilde r_{i^\star}
\leftarrow
r_{i^\star}+d_{i^\star}/\beta$
\IF{$\widetilde r_{i^\star}\geq u$}
    \STATE \textbf{return} $y_{i^\star}$
\ELSE
    \STATE Run the configured base-model GSI fallback
\ENDIF
\end{algorithmic}
\end{algorithm}

Recall the optimal distribution and the output law of clipped ExpBoN are denoted by $\pi_{\beta,C}^{\star}(\cdot \mid h)$ and $\widetilde\pi_{n,\beta,C}(\cdot\mid h)$, respectively. 
We give the following theoretical guarantee: 
\begin{theorem}
\label{thm:clipped_exp_gsi_bias}
For every $n\geq1$,
\begin{equation}
\label{eq:finite_n_plus_clipping_bias}
    \operatorname{TV}
    \left(
        \widetilde\pi_{n,\beta,C}(\cdot\mid h),
        \pi_{\beta,B}^{\star}(\cdot\mid h)
    \right)
    \leq
    \rho_{\beta,C}(h)^n
    +
    \tau_C(h)
\end{equation}
where $\tau_C(h) := \mathbb E_{\pi_{\beta,B}^{\star}(\cdot\mid h)} \left[ 1- \min\left\{1,e^{C-d(h,y)}\right\} \right]$. 
Consequently, for every bounded function $g$,
\begin{align}
\label{eq:finite_n_plus_clipping_utility}
&
\left|
    \mathbb E_{\widetilde\pi_{n,\beta,C}(\cdot\mid h)}
    [g(Y)]
    -
    \mathbb E_{\pi_{\beta,B}^{\star}(\cdot\mid h)}
    [g(Y)]
\right|
\nonumber\\
&\qquad\leq
\big(\sup_y g(y)-\inf_y g(y) \big)
\left(
    \rho_{\beta,C}(h)^n
    +
    \tau_C(h)
\right).
\end{align} 
\end{theorem}
Theorem~\ref{thm:clipped_exp_gsi_bias} separates the geometric finite-candidate error $\rho_{\beta,C}(h)^n$ from the nonvanishing clipping bias $\tau_C(h)$. Increasing $n$ reduces only the former, while increasing $C$ is required to reduce the distortion of the original GSI target.

\section{Experiments}
\label{sec:experiments}

\paragraph{Models and benchmarks.}
We evaluate ExpGSI on two draft--target pairs: the math-specialized
Qwen2.5-Math family~\citep{yang2024qwen2}
(\textsc{1.5B-Instruct} draft, \textsc{7B-Instruct} target) and the
general-purpose Qwen3 family~\citep{yang2025qwen3} (\textsc{1.7B}
draft, \textsc{14B} target, thinking mode disabled), both with
\textsc{Qwen2.5-Math-PRM-7B} as the process reward model
($r\in[0,1]$) and both matching the configurations of GSI's
evaluation~\citep{geuter2025guided}. The three benchmarks span
complementary domains, answer formats, and difficulty:
MATH500~\citep{lightman2024let} (in-domain competition math),
MMLU-STEM~\citep{hendrycks2020measuring} (multiple-choice STEM,
outside the PRM's training domain), and Minerva
Math~\citep{lewkowycz2022solving} (harder university-level
problems). We fix 400-problem subsets for the first two and use
Minerva's full 272 problems; identical problems and seeds across
methods and budgets enable paired comparisons.

\paragraph{Baselines.}
We compare against GSI~\citep{geuter2025guided} and the three baselines of its primary
evaluation: S-BoN($\pi_S$) and S-BoN($\pi_B$), stepwise soft
best-of-$n$ over candidates generated entirely by the draft or the
target policy (the cost floor and the quality reference,
respectively), and RSD~\citep{liao2025reward}, which keeps a draft
step if its reward clears a binary threshold and otherwise resamples
from $\pi_B$. Since original RSD is not defined for the shared
$n$-candidate protocol, we adopt the instantiation from GSI's
evaluation. All methods share the same PRM, candidate budgets,
evaluation subsets, and generation settings.

\paragraph{Default settings and metrics.}
All models are served as vLLM~\citep{kwon2023efficient} instances on
NVIDIA A100 GPUs. We report final-answer accuracy, wall-clock time
per reasoning step, acceptance (the fraction of steps resolved with
a draft candidate rather than a target-policy fallback), and
estimated TFLOPs per problem under the FLOPs accounting of
RSD~\citep{liao2025reward}: $2\times$ parameters per processed
token, summed over the forward passes actually executed. All methods
share the decoding hyperparameters of GSI; the only hyperparameter
ExpGSI adds, the clipping level $C$, is calibrated once and held
fixed across all benchmarks, budgets, and both model families. Full
implementation details are given in Appendix \ref{sec:implementation_details}.

\subsection{Main Results}
\label{sec:main_results}

\begin{table*}[t]
\centering
\small
\setlength{\tabcolsep}{4.2pt}
\renewcommand{\arraystretch}{0.95}
\begin{tabular}{@{}c c l c c c c@{}}
\toprule
Model family
& $n$
& Method
& Acc. (\%)
& Time/step (s)
& Accept. (\%)
& Est. TFLOPs/prob. \\
\midrule

\multirow{10}{*}{Qwen2.5-Math(1.5b/7b/7b)}
& \multirow{5}{*}{4}
& S-BoN ($\pi_S$)
& $55.9 \pm 0.5$
& $0.49 \pm 0.01$
& --
& $252 \pm 6$ \\

&
& RSD$^\dagger$
& $57.3 \pm 1.2$
& $0.58 \pm 0.01$
& $95.7 \pm 0.2$
& $257 \pm 7$ \\

&
& S-BoN ($\pi_B$)
& $62.1 \pm 0.4$
& $1.06 \pm 0.01$
& --
& $297 \pm 6$ \\

&
& GSI
& $59.9 \pm 0.6$
& $0.87 \pm 0.01$
& $82.1 \pm 0.2$
& $370 \pm 5$ \\

&
& \textbf{ExpGSI (ours)}
& $59.9 \pm 0.8$
& $\mathbf{0.84 \pm 0.02}$
& $81.9 \pm 0.4$
& $\mathbf{285 \pm 4}$ \\
\cmidrule(lr){2-7}

& \multirow{5}{*}{16}
& S-BoN ($\pi_S$)
& $58.1 \pm 0.2$
& $1.05 \pm 0.01$
& --
& $1182 \pm 16$ \\

&
& RSD$^\dagger$
& $58.5 \pm 1.0$
& $1.17 \pm 0.03$
& $97.3 \pm 0.2$
& $1164 \pm 37$ \\

&
& S-BoN ($\pi_B$)
& $62.5 \pm 0.2$
& $1.84 \pm 0.05$
& --
& $1258 \pm 13$ \\

&
& GSI
& $60.1 \pm 0.5$
& $1.71 \pm 0.03$
& $88.8 \pm 0.3$
& $1627 \pm 23$ \\

&
& \textbf{ExpGSI (ours)}
& $59.7 \pm 0.8$
& $\mathbf{1.40 \pm 0.01}$
& $88.7 \pm 0.4$
& $\mathbf{999 \pm 24}$ \\
\midrule

\multirow{10}{*}{Qwen3
(1.7b/14b/7b)}
& \multirow{5}{*}{4}
& S-BoN ($\pi_S$)
& $54.5 \pm 0.5$
& $0.28 \pm 0.00$
& --
& $521 \pm 1$ \\

&
& RSD$^\dagger$
& $55.8 \pm 0.2$
& $0.31 \pm 0.00$
& $97.7 \pm 0.0$
& $541 \pm 3$ \\

&
& S-BoN ($\pi_B$)
& $66.0 \pm 0.5$
& $0.78 \pm 0.00$
& --
& $642 \pm 8$ \\

&
& GSI
& $61.2 \pm 0.3$
& $0.60 \pm 0.00$
& $89.4 \pm 0.1$
& $1202 \pm 9$ \\

&
& \textbf{ExpGSI (ours)}
& $60.4 \pm 0.0$
& $\mathbf{0.55 \pm 0.00}$
& $89.5 \pm 0.0$
& $\mathbf{891 \pm 2}$ \\
\cmidrule(lr){2-7}

& \multirow{5}{*}{16}
& S-BoN ($\pi_S$)
& $54.8 \pm 0.1$
& $0.71 \pm 0.00$
& --
& $2183 \pm 4$ \\

&
& RSD$^\dagger$
& $56.9 \pm 1.4$
& $0.74 \pm 0.00$
& $98.5 \pm 0.0$
& $2203 \pm 19$ \\

&
& S-BoN ($\pi_B$)
& $65.9 \pm 0.0$
& $1.34 \pm 0.00$
& --
& $2565 \pm 34$ \\

&
& GSI
& $61.0 \pm 0.1$
& $1.50 \pm 0.01$
& $92.7 \pm 0.1$
& $5769 \pm 40$ \\

&
& \textbf{ExpGSI (ours)}
& $61.8 \pm 0.6$
& $\mathbf{1.05 \pm 0.00}$
& $93.0 \pm 0.3$
& $\mathbf{3169 \pm 6}$ \\
\bottomrule
\end{tabular}
\caption{
Macro-average results over MATH500, MMLU-STEM, and Minerva Math and multiple seeds.
ExpGSI maintains accuracy and acceptance comparable to GSI while reducing time and computation. 
At $n=16$ it reduces time and TFLOPs by $18\%$ and $39\%$ for Qwen2.5-Math, and by $30\%$ and $45\%$ for Qwen3. 
% Values are over $3$ seeds for Qwen2.5-Math and $2$ seeds for Qwen3. 
Bold values indicate ExpGSI efficiency improvements over GSI. $^\dagger$RSD follows the GSI implementation and
hyperparameter configuration.
}
\label{tab:model_family_average}
\end{table*}

\paragraph{ExpGSI is consistently cheaper than GSI, and the margin grows with $n$.}
Figure~\ref{fig:latency} summarizes our main result on the
compute--accuracy plane, one panel per candidate budget: at every
budget, ExpGSI reaches GSI's accuracy at strictly lower compute. The
gap widens steadily with the budget: ExpGSI saves $14\%$, $23\%$,
$31\%$, and $39\%$ of GSI's estimated compute at $n=2$, $4$, $8$,
and $16$, with time per step dropping by $18\%$ at $n{=}16$
(Table~\ref{tab:model_family_average}; per-benchmark tables in
Appendix~\ref{sec:Additional experimental results}).  At $n{=}16$ the compounding is strong enough that ExpGSI's cost falls below even RSD ($999$ vs. $1164$ TFLOPs per
problem) and the draft-only S-BoN($\pi_S$) ($1182$), while scoring
$1.2$--$1.6$ points higher than either.

\paragraph{ExpGSI has accuracy matched to GSI.}
The acceleration costs essentially no accuracy. Across the four
operating points in Table~\ref{tab:model_family_average},
macro-average accuracy differs from GSI by at most $0.8$ points,
with no consistent sign: ExpGSI is ahead at Qwen3 $n{=}16$ and equal
or marginally behind elsewhere, all within seed-level variation.
Acceptance rates agree to within $0.3$ points throughout, so the
early exit leaves the accept/fallback behavior of the underlying
sampler essentially untouched. Such lossless acceleration is not a
given: RSD, the existing reward-gated accelerator, buys its speed
with accuracy, scoring $1.2$--$4.9$ macro points below ExpGSI across
the four operating points. Among all baselines, only S-BoN($\pi_B$)
is more accurate than the tilted pair, and it occupies a different
design point: every token is generated by the target policy, at the
highest step latency of all methods ($1.84$\,s vs. ExpGSI's $1.40$\,s at Qwen2.5, $n{=}16$). ExpGSI thus retains the design point of GSI--draft-side generation with target-aware selection--while making it strictly cheaper.

\paragraph{Generalization to Qwen3.}
We repeat the comparison with the general-purpose
Qwen3~\citep{yang2025qwen3} pair at $n\in\{4,16\}$ with two random
seeds, reusing the clipping level as-is--no additional tuning. As
Table~\ref{tab:model_family_average} shows, the picture carries
over: accuracy and acceptance stay at GSI's level, while the
$n{=}16$ compute saving grows from $39\%$ (Qwen2.5) to $45\%$,
consistent with the cost structure of early exit--the larger the
target, the larger the share of GSI's cost spent on the avoided
evaluations. The result also shows that ExpGSI does not rely on
math-specialized policies.

\subsection{Ablation and Robustness Analysis}
\label{sec:analysis}

\paragraph{Robustness to the reward model.}
We replace the 7B math-specialized PRM with the smaller,
general-purpose \textsc{Skywork-o1-Open-PRM-Qwen-2.5-1.5B}%
~\citep{team2024skywork}--the reference PRM of
RSD~\citep{liao2025reward}. The conclusions carry over:
accuracy and acceptance stay at GSI's level, ExpGSI still reduces
the estimated computation of GSI, and it remains ahead of RSD in
accuracy. Complete results are given in
Appendix~\ref{sec:Additional experimental results}.

\paragraph{Sensitivity to the clipping level.}
The clipping level $C$ is the only hyperparameter ExpGSI adds,
trading fidelity to the GSI selection rule against early-exit
efficiency. Sweeping $C$ on MATH500 at $n{=}16$ leaves accuracy
within seed-level noise for $C$ between $0.5\times$ and $1.5\times$
the calibrated value, with a mild drop only at $2\times$, while
estimated computation increases with $C$; $C{=}\infty$ disables
early exit and recovers the full-scan cost. Detailed results are in
Appendix~\ref{sec:Additional experimental results}.

\section{Conclusion and Future Work}
We introduced ExpBoN, an exponential-noise alternative to soft Best-of-$n$ sampling for inference-time alignment.
Its exact finite-$n$ tilted decomposition yields geometric convergence to the target distribution and desirable regret guarantees.
We further integrated ExpBoN into guided speculative inference, resulting in ExpGSI, which substantially accelerates LLM alignment, as validated by experiments.

% \begin{table}[t]
% \centering
% \small
% \caption{Comparison of SBoN and ExpBoN. Both target the same
% exponentially tilted distribution; ExpBoN provides additional
% finite-budget structure.}
% \label{tab:sbon_expbon}
% \setlength{\tabcolsep}{4pt}
% \resizebox{\columnwidth}{!}{
% \begin{tabular}{lll}
% \toprule
% \textbf{Property} & \textbf{SBoN} & \textbf{ExpBoN (ours)} \\
% \midrule
% Perturbation & Gumbel & Exponential \\
% Finite-$n$ structure & implicit output law & exact tilted-mixture
% decomposition \\
% Convergence to tilt & $O(1/n)$ bounds & geometric TV and
% bidirectional KL \\
% Sequential evaluation & no exact early exit & exact early exit via
% first threshold passer \\
% \bottomrule
% \end{tabular}}
% \end{table}

% \section{Conclusion and Future Work}

% We introduced ExpBoN, an exponential-noise alternative to soft Best-of-$n$ sampling for inference-time alignment.
% Its exact finite-$n$ tilted decomposition yields geometric convergence to the target distribution and desirable regret guarantees.
% We further integrated ExpBoN into guided speculative inference, resulting in ExpGSI, which substantially accelerates LLM alignment, as validated by experiments.

We discuss several future directions.
First, the connection between ExpBoN and Permute-and-Flip~\citep{zhao2024permute} opens a natural path toward jointly performing alignment and watermarking, whose tradeoff has recently been investigated by~\citet{verma2026watermarking}.
Second, under proxy-reward misspecification, faster convergence may approach an overoptimized target more rapidly, and it is therefore useful to combine ExpBoN with reward-hacking mitigation methods~\citep{khalaf2026inference}.
Finally, SBoN has recently been employed in diffusion language models~\citep{bu2026dprm}, and it is of interest to investigate the use of ExpBoN in this setting as well.

\newpage 
\bibliography{ref}
\newpage

\section{Appendices}

\subsection{Proofs for Section~\ref{subsec::decomp_converge}}

We define
\begin{equation*}
    s(x):=\frac{r(x)}{\lambda},
    \qquad
    s_{\max}:=\frac{r_{\max}}{\lambda},
\end{equation*}
and, for $t\in\mathbb R$, 
\begin{equation*}
    B_t
    :=
    \left\{
        x\in\mathcal X_P:s(x)\leq t
    \right\}.
\end{equation*}

\begin{lemma}
\label{lem:ExpBoN_conditional_tilt}
For almost every $t$ with respect to the distribution of $M_n$, and every
$x\in\mathcal X_P$,
\begin{equation*}
    \Pr(Y=x\mid M_n=t)
    =
    \frac{
        P(x)e^{s(x)}\mathbf 1\{x\in B_t\}
    }{
        \sum_{z\in\mathcal X_P}
        P(z)e^{s(z)}\mathbf 1\{z\in B_t\}
    }.
\end{equation*}
Thus, conditional on $M_n=t$, the output follows $P_\lambda^\star$ restricted to $B_t$ and renormalized.  In particular, for almost every $t\geq s_{\max}$,
\begin{equation*}
    \Pr(Y=x\mid M_n=t)
    =
    P_\lambda^\star(x).
\end{equation*}
\end{lemma}

\begin{proof}
Let $T=s(X)+E$ and let $F_T$ denote its distribution function.  Since
$E\sim\operatorname{Exp}(1)$, for every $x\in\mathcal X_P$,
\begin{align*}
    \Pr(X=x,T\in dt)
    &=
    P(x)e^{-(t-s(x))}
    \mathbf 1\{t\geq s(x)\}\,dt\\
    &=
    e^{-t}P(x)e^{s(x)}
    \mathbf 1\{t\geq s(x)\}\,dt.
\end{align*}
The noisy scores are independent and continuously distributed, so their maximum is attained at a unique index almost surely.  By exchangeability,
\begin{align*}
    \Pr(Y=x,M_n\in dt)
    ={}&
    n e^{-t}P(x)e^{s(x)}
    \mathbf 1\{t\geq s(x)\}\\
    &\times F_T(t)^{n-1}\,dt.
\end{align*}
For fixed $t$, all factors except
$P(x)e^{s(x)}\mathbf 1\{s(x)\leq t\}$ are independent of $x$.
Normalizing over $x\in\mathcal X_P$ proves the result.
\end{proof}

\subsubsection{Proof of Theorem~\ref{thm:ExpBoN_decomposition}}

\begin{proof}
For one noisy score,
\begin{align*}
    \Pr(T<s_{\max})
    &=
    \mathbb E_P
    \left[
        1-e^{-(s_{\max}-s(X))}
    \right]\\
    &=
    1-
    \mathbb E_P
    \left[
        e^{s(X)-s_{\max}}
    \right]\\
    &=
    \rho_\lambda.
\end{align*}
Since the $T_i$ are independent,
\begin{equation*}
    \Pr(M_n<s_{\max})
    =
    \rho_\lambda^n.
\end{equation*}
If $\rho_\lambda=0$, then $r(X)=r_{\max}$ $P$-almost surely, so $\widetilde P_{n,\lambda}=P_\lambda^\star=P$, and the claimed decomposition holds for any distribution $Q_{n,\lambda}$ on $\mathcal X_P$.  Hence, for the remainder of the proof, suppose that $\rho_\lambda>0$. 
By Lemma~\ref{lem:ExpBoN_conditional_tilt}, conditional on $M_n\geq s_{\max}$, the output follows $P_\lambda^\star$.  Define
\begin{equation*}
    Q_{n,\lambda}(A)
    :=
    \Pr(Y\in A\mid M_n<s_{\max}),
    \qquad
    A\subseteq\mathcal X_P.
\end{equation*}
Conditioning on $\{M_n\geq s_{\max}\}$ and
$\{M_n<s_{\max}\}$ gives
\begin{equation*}
    \widetilde P_{n,\lambda}
    = \left(
        1-\rho_\lambda^n
    \right)
    P_\lambda^\star +
    \rho_\lambda^n Q_{n,\lambda}.
\end{equation*}

We also record the property needed for the lower bounds below.  
Let
\begin{equation*}
    A_\lambda
    :=
    \{x\in\mathcal X_P:r(x)=r_{\max}\}.
\end{equation*}
If $X_i\in A_\lambda$, then
\begin{equation*}
    T_i=s_{\max}+E_i\geq s_{\max}
\end{equation*}
almost surely.  Hence, on $\{M_n<s_{\max}\}$, none of the candidates, and
therefore not the selected output, belongs to $A_\lambda$.  Thus,
\begin{equation*}
    Q_{n,\lambda}(A_\lambda)=0.
\end{equation*}
\end{proof}

\subsubsection{Proof of Theorem~\ref{thm:pf_sbon_kl_tv}}

\begin{proof}
Since $r$ is nonconstant on $\mathcal X_P$,
$\rho_\lambda\in(0,1)$, while $q_\lambda>0$ and
$p_{\min,\lambda}>0$. 
Write
\begin{equation*}
    a:=\rho_\lambda^n,
    \qquad
    P^\star:=P_\lambda^\star,
    \qquad
    Q:=Q_{n,\lambda}.
\end{equation*}
Theorem~\ref{thm:ExpBoN_decomposition} gives
\begin{equation*}
    \widetilde P_{n,\lambda}
    =
    (1-a)P^\star+aQ.
\end{equation*}
Therefore,
\begin{equation*}
    \mathrm{TV}
    \big(
        P^\star,
        \widetilde P_{n,\lambda}
    \big)
    =
    a\,\mathrm{TV}(P^\star,Q)
    \leq
    a.
\end{equation*}
Moreover, since $Q(A_\lambda)=0$,
\begin{align*}
&P^\star(A_\lambda)
-
\widetilde P_{n,\lambda}(A_\lambda)\\
&\qquad=
q_\lambda
-
(1-a)q_\lambda
=
q_\lambda a.
\end{align*}
Taking the event $A_\lambda$ in the variational definition of total
variation yields
\begin{equation*}
    \mathrm{TV}
    \big(
        P^\star,
        \widetilde P_{n,\lambda}
    \big)
    \geq
    q_\lambda a.
\end{equation*}
Pinsker's inequality in either direction then gives
\begin{align*}
    D_{\mathrm{KL}}
    \big(
        P^\star
        \Vert
        \widetilde P_{n,\lambda}
    \big)
    &\geq
    2q_\lambda^2a^2,\\
    D_{\mathrm{KL}}
    \big(
        \widetilde P_{n,\lambda}
        \Vert
        P^\star
    \big)
    &\geq
    2q_\lambda^2a^2.
\end{align*}

For the first upper bound in the target-to-sampler direction,
\begin{equation*}
    \widetilde P_{n,\lambda}(x)
    \geq
    (1-a)P^\star(x),
\end{equation*}
and hence
\begin{equation*}
    D_{\mathrm{KL}}
    \big(
        P^\star
        \Vert
        \widetilde P_{n,\lambda}
    \big)
    \leq
    -\log(1-a).
\end{equation*}
For the sharper finite-alphabet bounds, define
\begin{equation*}
    \chi^2(Q_1\Vert Q_2)
    :=
    \sum_x
    \frac{
        (Q_1(x)-Q_2(x))^2
    }{
        Q_2(x)
    }.
\end{equation*}
Using $D_{\mathrm{KL}}(Q_1\Vert Q_2)\leq
\chi^2(Q_1\Vert Q_2)$,
\begin{align*}
&D_{\mathrm{KL}}
\big(
    P^\star
    \Vert
    \widetilde P_{n,\lambda}
\big)\\
&\quad\leq
\chi^2
\big(
    P^\star
    \Vert
    \widetilde P_{n,\lambda}
\big)\\
&\quad=
\sum_x
\frac{
    a^2(P^\star(x)-Q(x))^2
}{
    (1-a)P^\star(x)+aQ(x)
}\\
&\quad\leq
\frac{a^2}{1-a}
\chi^2(Q\Vert P^\star).
\end{align*}
Furthermore,
\begin{align*}
    \chi^2(Q\Vert P^\star)
    &=
    \sum_x
    \frac{Q(x)^2}{P^\star(x)}-1\\
    &\leq
    p_{\min,\lambda}^{-1}-1.
\end{align*}
This proves the second target-to-sampler KL upper bound.

Finally,
\begin{equation*}
    \chi^2
    \big(
        \widetilde P_{n,\lambda}
        \Vert
        P^\star
    \big)
    =
    a^2\chi^2(Q\Vert P^\star).
\end{equation*}
The inequality
$D_{\mathrm{KL}}(Q_1\Vert Q_2)
\leq\log(1+\chi^2(Q_1\Vert Q_2))$ therefore gives
\begin{equation*}
    D_{\mathrm{KL}}
    \big(
        \widetilde P_{n,\lambda}
        \Vert
        P^\star
    \big)
    \leq
    \log\left(
        1+
        \left(
            p_{\min,\lambda}^{-1}-1
        \right)
        a^2
    \right).
\end{equation*}
Substituting $a=\rho_\lambda^n$ proves all finite-$n$ bounds.  Since the
constants $q_\lambda$ and $p_{\min,\lambda}$ are fixed and positive, the
stated asymptotic orders follow.
\end{proof}

\subsection{Proofs for Section~\ref{subsec:ExpBoN_reward_convergence}}

\subsubsection{Proof of Theorem~\ref{thm:ExpBoN_reward_objective}}

\begin{proof}
By Lemma~\ref{lem:ExpBoN_conditional_tilt}, conditional on $M_n=t$, the selected output follows $P_\lambda^\star$ restricted to $\{x:s(x)\leq t\}$.  As $t$ increases, only points with larger reward are added to the conditioning set.  Hence the conditional expected reward is nondecreasing in $t$ and is at most $\mathbb E_{P_\lambda^\star}[r(X)]$. Under the natural coupling $M_{n+1}=\max\{M_n,T_{n+1}\}\geq M_n$ almost surely.  Since the preceding conditional expected reward is a nondecreasing function of $t$, $\mathbb E_{\widetilde P_{n,\lambda}}[r(X)]$ is nondecreasing in $n$ and is upper bounded by the tilted-target reward.

By Theorem~\ref{thm:ExpBoN_decomposition},
\begin{align*}
&\mathbb E_{P_\lambda^\star}[r(X)]
-
\mathbb E_{\widetilde P_{n,\lambda}}[r(X)]\\
&\qquad=
\rho_\lambda^n
\left(
    \mathbb E_{P_\lambda^\star}[r(X)]
    -
    \mathbb E_{Q_{n,\lambda}}[r(X)]
\right).
\end{align*}
The expression in parentheses lies in $[0,\Delta_r]$, proving the displayed bound.  
Since $r$ is not constant, $\rho_\lambda\in(0,1)$, and convergence follows.
\end{proof}

\subsubsection{Proof of Corollary~\ref{cor:ExpBoN_relative_reward}}

\begin{proof}
Since $r$ is nonconstant and $P_\lambda^\star$ has full support on
$\mathcal X_P$,
\begin{equation*}
    \mathbb E_{P_\lambda^\star}[\overline r(X)]>0.
\end{equation*}
The exact decomposition gives
\begin{align*}
&\mathbb E_{P_\lambda^\star}[\overline r(X)]
-
\mathbb E_{\widetilde P_{n,\lambda}}[\overline r(X)]\\
&\qquad=
\rho_\lambda^n
\left(
    \mathbb E_{P_\lambda^\star}[\overline r(X)]
    -
    \mathbb E_{Q_{n,\lambda}}[\overline r(X)]
\right).
\end{align*}
The left-hand side is nonnegative by
Theorem~\ref{thm:ExpBoN_reward_objective}.  Since $\overline r\geq0$,
\begin{equation*}
    \mathbb E_{P_\lambda^\star}[\overline r(X)]
    -
    \mathbb E_{Q_{n,\lambda}}[\overline r(X)]
    \leq
    \mathbb E_{P_\lambda^\star}[\overline r(X)].
\end{equation*}
Dividing by the positive denominator proves the result.
\end{proof}

\subsubsection{Proof of Corollary~\ref{cor:ExpBoN_reference_kl_localization}}

\begin{proof}
For every distribution $Q$ on $\mathcal X_P$,
\begin{align*}
&D_{\mathrm{KL}}(Q\Vert P)
-
D_{\mathrm{KL}}(P_\lambda^\star\Vert P)\\
&\qquad=
D_{\mathrm{KL}}(Q\Vert P_\lambda^\star)
+
\frac{
    \mathbb E_Q[r(X)]
    -
    \mathbb E_{P_\lambda^\star}[r(X)]
}{\lambda}.
\end{align*}
Setting $Q=\widetilde P_{n,\lambda}$, using the nonnegativity of KL and
Theorem~\ref{thm:ExpBoN_reward_objective}, gives
\begin{align*}
&D_{\mathrm{KL}}
\big(
    \widetilde P_{n,\lambda}
    \Vert
    P
\big)
-
D_{\mathrm{KL}}
\big(
    P_\lambda^\star
    \Vert
    P
\big)\\
&\qquad\geq
-\frac{\Delta_r}{\lambda}\rho_\lambda^n.
\end{align*}
The same reward theorem also gives
\begin{equation*}
    \mathbb E_{\widetilde P_{n,\lambda}}[r(X)]
    -
    \mathbb E_{P_\lambda^\star}[r(X)]
    \leq
    0.
\end{equation*}
Therefore, Theorem~\ref{thm:pf_sbon_kl_tv} yields
\begin{align*}
&D_{\mathrm{KL}}
\big(
    \widetilde P_{n,\lambda}
    \Vert
    P
\big)
-
D_{\mathrm{KL}}
\big(
    P_\lambda^\star
    \Vert
    P
\big)\\
&\qquad\leq
D_{\mathrm{KL}}
\big(
    \widetilde P_{n,\lambda}
    \Vert
    P_\lambda^\star
\big)\\
&\qquad\leq
\log\left(
    1+
    \left(
        p_{\min,\lambda}^{-1}-1
    \right)
    \rho_\lambda^{2n}
\right),
\end{align*}
which proves the localization bound.
\end{proof}

\subsubsection{Proof of the distribution-free reference-KL bound}

\begin{proof}
Let
\begin{equation*}
    a_\lambda:=\frac{\Delta_r}{\lambda},
    \qquad
    Z(x):=
    \frac{
        \widetilde P_{n,\lambda}(x)
    }{
        P(x)
    }.
\end{equation*}
By exchangeability,
\begin{equation*}
    Z(x)
    =
    n\Pr\left(
        s(x)+E_1
        \geq
        \max_{j=2,\ldots,n}
        \{s(X_j)+E_j\}
    \right).
\end{equation*}
The winning probability of the fixed candidate is maximized by assigning it
reward $r_{\max}$ and assigning every competitor reward $r_{\min}$.  Hence
\begin{align*}
    Z(x)
    &\leq
    n\int_0^\infty
    e^{-e}
    \left(
        1-e^{-a_\lambda-e}
    \right)^{n-1}\,de\\
    &=
    e^{a_\lambda}
    \left[
        1-
        \left(
            1-e^{-a_\lambda}
        \right)^n
    \right]
    =:
    u_{n,\lambda}.
\end{align*}
Similarly, the winning probability is minimized by assigning the fixed candidate reward $r_{\min}$ and every competitor reward $r_{\max}$, which gives
\begin{align*}
    Z(x)
    &\geq
    n\int_{a_\lambda}^\infty
    e^{-e}
    \left(
        1-e^{-(e-a_\lambda)}
    \right)^{n-1}\,de\\
    &=
    e^{-a_\lambda}.
\end{align*}
Moreover, $\mathbb E_P[Z(X)]=1$, $u_{n,\lambda}\leq n$, and $u_{n,\lambda}\leq e^{a_\lambda}$.
Therefore,
\begin{equation*}
    D_{\mathrm{KL}}
    \big(
        \widetilde P_{n,\lambda}
        \Vert
        P
    \big)
    =
    \mathbb E_P[Z(X)\log Z(X)]
    \leq
    \log u_{n,\lambda}
    \leq
    \log n.
\end{equation*}
Also, $Z(X)\in[e^{-a_\lambda},e^{a_\lambda}]$.  Applying the chord bound for
the convex function $z\mapsto z\log z$ under the constraint
$\mathbb E_P[Z(X)]=1$ gives
\begin{equation*}
    \mathbb E_P[Z(X)\log Z(X)]
    \leq
    a_\lambda
    \tanh\left(
        \frac{a_\lambda}{2}
    \right)
    \leq
    \frac{a_\lambda^2}{2}.
\end{equation*}
Combining the two bounds proves
\begin{equation*}
    D_{\mathrm{KL}}
    \big(
        \widetilde P_{n,\lambda}
        \Vert
        P
    \big)
    \leq
    \min\left\{
        \log n,
        \frac{\Delta_r^2}{2\lambda^2}
    \right\}.
\end{equation*}
\end{proof}

\subsection{Proof of Theorem~\ref{thm:ExpBoN_utility_dominance}}
\begin{proof}
Conditional on $X_{1:n}=x_{1:n}$, ExpBoN is report-noisy-max with independent exponential noise and is therefore distributionally equivalent to Permute-and-Flip~\citep{ding2021permute}.  Conditional SBoN is softmax sampling applied to the same candidate reward vector at the same temperature $\lambda$. The ``never worse'' property of Permute-and-Flip \citep[Theorem~3.1(3)]{zhao2024permute}, equivalently \citet[Theorem~2]{mckenna2020permute}, therefore gives the conditional inequality. Averaging over $X_{1:n}$ gives the marginal result.
\end{proof}

\subsection{Proofs for Section~\ref{subsec:regret_analysis}}
\label{subsec:appendix_regret_analysis}

We first record the mechanism-level stability result used below.

\begin{lemma}
\label{lem:exp_noise_stability}
For deterministic score vectors $u,v\in\mathbb R^n$, let $q_u$ and $q_v$
denote the distributions of
\begin{equation*}
    \arg\max_{i=1,\ldots,n}\{u_i+E_i\}
    \quad\text{and}\quad
    \arg\max_{i=1,\ldots,n}\{v_i+E_i\},
\end{equation*}
respectively, where the $E_i$ are i.i.d.
$\operatorname{Exp}(1)$.  Define
\begin{equation*}
    D
    :=
    \max_i(v_i-u_i)
    -
    \min_i(v_i-u_i).
\end{equation*}
Then, for every $i$,
\begin{equation*}
    e^{-D}q_u(i)
    \leq
    q_v(i)
    \leq
    e^Dq_u(i).
\end{equation*}
Consequently, in either KL direction,
\begin{equation*}
    D_{\mathrm{KL}}(q_u\Vert q_v)
    \leq
    D\tanh\left(\frac D2\right),
\end{equation*}
and
\begin{equation*}
    \mathrm{TV}(q_u,q_v)
    \leq
    \tanh\left(\frac D2\right).
\end{equation*}
\end{lemma}

\begin{proof}
Fix $i$ and subtract the common constant $v_i-u_i$ from all coordinates of
$v$, which does not change $q_v$.  The $i$th score then equals $u_i$, while
each competitor score differs from the corresponding score under $u$ by a
quantity in $[-D,D]$.  Monotonicity of the winning probability in each
competitor score gives
\begin{equation*}
    q_{u-De_i}(i)
    \leq
    q_v(i)
    \leq
    q_{u+De_i}(i),
\end{equation*}
where $e_i$ is the $i$th standard basis vector.

Let $F_E$ be the distribution function of an
$\operatorname{Exp}(1)$ variable.  Then
\begin{align*}
    q_{u+De_i}(i)
    &=
    \int_0^\infty
    e^{-e}
    \prod_{j\neq i}
    F_E(e+D+u_i-u_j)\,de\\
    &=
    e^D
    \int_D^\infty
    e^{-t}
    \prod_{j\neq i}
    F_E(t+u_i-u_j)\,dt\\
    &\leq
    e^Dq_u(i).
\end{align*}
Similarly,
\begin{align*}
    q_{u-De_i}(i)
    &=
    e^{-D}
    \int_{-D}^\infty
    e^{-t}
    \prod_{j\neq i}
    F_E(t+u_i-u_j)\,dt\\
    &\geq
    e^{-D}q_u(i).
\end{align*}
This proves the likelihood-ratio bound.  The KL and total-variation bounds
follow by applying the chord bounds to a likelihood ratio constrained to
$[e^{-D},e^D]$ and having mean one.
\end{proof}

\subsubsection{Proof of Theorem~\ref{thm:ExpBoN_proxy_stability}}

\begin{proof}
Condition on the candidate tuple $X_1,\ldots,X_n$ and apply
Lemma~\ref{lem:exp_noise_stability} to
\begin{equation*}
    u_i=\frac{r_p(X_i)}{\lambda},
    \qquad
    v_i=\frac{r_t(X_i)}{\lambda}.
\end{equation*}
The conditional score perturbation has oscillation at most
$\omega_d/\lambda$, giving the range-based conditional KL and
TV bounds.
Since the candidate marginal is $P^n$ under both mechanisms, the KL chain rule gives \begin{equation*} D_{\mathrm{KL}}(P^n q_t\Vert P^n q_p) = \mathbb E_{P^n} \left[ D_{\mathrm{KL}}\big(q_t(\cdot\mid X_{1:n}) \Vert q_p(\cdot\mid X_{1:n})\big) \right], \end{equation*} and the analogous total-variation distance is the $P^n$-average of the conditional total-variation distances. 
The returned response is a deterministic function of this joint
object, so data processing gives the range-based output-policy bounds.

For the variance-sensitive bound, Lemma~\ref{lem:exp_noise_stability} and
$\tanh(z)\leq z$ give the conditional KL bound
\begin{equation*}
    \frac{1}{2\lambda^2}
    \left(
        \max_i d(X_i)-\min_i d(X_i)
    \right)^2.
\end{equation*}
Writing $d_i=d(X_i)$ and
$\overline d_n=n^{-1}\sum_{i=1}^n d_i$,
\begin{equation*}
    \left(
        \max_i d_i-\min_i d_i
    \right)^2
    \leq
    2\sum_{i=1}^n
    (d_i-\overline d_n)^2.
\end{equation*}
Taking expectations and using
\begin{equation*}
    \mathbb E
    \left[
        \sum_{i=1}^n
        (d_i-\overline d_n)^2
    \right]
    =
    (n-1)\operatorname{Var}_P(d(X))
\end{equation*}
proves the variance-sensitive KL bound in both directions.  Pinsker's
inequality gives the corresponding TV bound.
\end{proof}

\subsubsection{Proof of Theorem~\ref{thm:ExpBoN_direct_regret}}

\begin{proof}
The log-likelihood ratio between the true- and proxy-tilted targets satisfies
\begin{equation*}
    \log
    \frac{
        P_{\lambda,t}^\star(x)
    }{
        P_{\lambda,p}^\star(x)
    }
    =
    \frac{d(x)}{\lambda}
    +
    \textnormal{constant},
\end{equation*}
and therefore has oscillation $\omega_d/\lambda$.  
For completeness, if a likelihood ratio $L=dP/dQ$ satisfies $\operatorname{osc}(\log L)\leq w$, then the chord bound for $L$ under $\mathbb E_Q[L]=1$, optimized over its possible endpoints, gives $\mathrm{TV}(P,Q)\leq\tanh(w/4)$.
The sharp
likelihood-ratio range bound gives
\begin{equation*}
    \mathrm{TV}
    \big(
        P_{\lambda,t}^\star,
        P_{\lambda,p}^\star
    \big)
    \leq
    \tanh\left(
        \frac{\omega_d}{4\lambda}
    \right).
\end{equation*}
The exact ExpBoN decomposition gives
\begin{equation*}
    \mathrm{TV}
    \big(
        P_{\lambda,p}^\star,
        \widetilde P_{n,\lambda}^{p}
    \big)
    \leq
    \rho_{\lambda,p}^{\,n}.
\end{equation*}
Therefore,
\begin{align*}
    \operatorname{Reg}_{n,\lambda}
    ={}&
    B_t(\lambda)
    +
    \mathbb E_{P_{\lambda,t}^\star}[r_t(X)]
    -
    \mathbb E_{\widetilde P_{n,\lambda}^{p}}[r_t(X)]\\
    \leq{}&
    B_t(\lambda)
    +
    \Delta_t
    \mathrm{TV}
    \big(
        P_{\lambda,t}^\star,
        \widetilde P_{n,\lambda}^{p}
    \big),
\end{align*}
and the triangle inequality proves the claim.
\end{proof}

\subsubsection{Proof of Theorem~\ref{thm:ExpBoN_coverage_regret}}

\begin{proof}
For $u\in\{p,t\}$, write
\begin{equation*}
    D_u
    :=
    D_{\mathrm{KL}}
    \big(
        P_{\lambda,u}^\star
        \Vert
        P
    \big),
\end{equation*}
and define
\begin{equation*}
    A_u
    :=
    \left\{
        x\in\mathcal X_P:r_u(x)=r_{u,\max}
    \right\},
    \qquad
    C_{\infty,u}:=\frac{1}{P(A_u)}.
\end{equation*}
For any $c\in\mathbb R$, adding $c$ to $r_p$ changes neither
$P_{\lambda,p}^\star$ nor $\widetilde P_{n,\lambda}^{p}$ and replaces $d$
by $d-c$.  The optimality of $P_{\lambda,p}^\star$ for the proxy
KL-regularized objective gives
\begin{equation*}
    \mathbb E_{P_{\lambda,t}^\star}[r_p(X)]
    -
    \mathbb E_{P_{\lambda,p}^\star}[r_p(X)]
    \leq
    \lambda(D_t-D_p).
\end{equation*}
Moreover, Cauchy--Schwarz gives
\begin{align*}
&\left|
    \mathbb E_{P_{\lambda,u}^\star}[d(X)-c]
\right|\\
&\quad\leq
\left(
    \sum_{x\in \mathcal{X}_P}
    \frac{
        P_{\lambda,u}^\star(x)^2
    }{
        P(x)
    }
\right)^{1/2}
\left(
    \mathbb E_P[(d(X)-c)^2]
\right)^{1/2}.
\end{align*}
Since
\begin{equation*}
    \frac{
        P_{\lambda,u}^\star(x)
    }{
        P(x)
    }
    =
    \frac{
        e^{(r_u(x)-r_{u,\max})/\lambda}
    }{
        \mathbb E_P
        \left[
            e^{(r_u(X)-r_{u,\max})/\lambda}
        \right]
    }
    \leq
    \frac{1}{P(A_u)}
    =
    C_{\infty,u},
\end{equation*}
we have
\begin{equation*}
    \sum_{x\in \mathcal{X}_P}
    \frac{
        P_{\lambda,u}^\star(x)^2
    }{
        P(x)
    }
    \leq
    C_{\infty,u}.
\end{equation*}
Jensen's inequality further gives
\begin{equation*}
    \mathbb E_P[(d(X)-c)^2]
    \leq
    \lambda
    \log
    \mathbb E_P
    \left[
        e^{(d(X)-c)^2/\lambda}
    \right].
\end{equation*}
For every $c\in\mathbb R$,
\begin{align*}
&\mathbb E_{P_{\lambda,t}^\star}[r_t(X)]
-
\mathbb E_{P_{\lambda,p}^\star}[r_t(X)]\\
& =
\mathbb E_{P_{\lambda,t}^\star}[r_p(X)]
-
\mathbb E_{P_{\lambda,p}^\star}[r_p(X)] \\
& \qquad
+
\mathbb E_{P_{\lambda,t}^\star}[d(X)-c]
-
\mathbb E_{P_{\lambda,p}^\star}[d(X)-c].
\end{align*}
Combining this identity with the preceding bounds and the optimality
inequality for the proxy objective, optimizing over $c$ yields
\begin{align*}
&\mathbb E_{P_{\lambda,t}^\star}[r_t(X)]
-
\mathbb E_{P_{\lambda,p}^\star}[r_t(X)]\\
&\qquad\leq
\lambda(D_t-D_p)
+
\sqrt{\overline\varepsilon_\lambda}
\left(
    \sqrt{C_{\infty,t}}
    +
    \sqrt{C_{\infty,p}}
\right).
\end{align*}

Let $P_{\infty,t}^\star=P(\cdot\mid A_t)$.  
Since
$\mathbb E_{P_{\infty,t}^\star}[r_t(X)]=r_{t,\max}$ and
$D_{\mathrm{KL}}(P_{\infty,t}^\star\Vert P)=\log C_{\infty,t}$,
comparing $P_{\lambda,t}^\star$ with $P_{\infty,t}^\star$ in the true
KL-regularized objective gives
\begin{equation*}
    B_t(\lambda)
    \leq
    \lambda
    \left(
        \log C_{\infty,t}
        -
        D_t
    \right).
\end{equation*}
Finally, the exact ExpBoN decomposition gives
\begin{equation*}
    \mathbb E_{P_{\lambda,p}^\star}[r_t(X)]
    -
    \mathbb E_{\widetilde P_{n,\lambda}^{p}}[r_t(X)]
    \leq
    R_{\max}\rho_{\lambda,p}^{\,n},
\end{equation*}
and
\begin{align*}
    \rho_{\lambda,p}
    &=
    1-
    \mathbb E_P
    \left[
        e^{(r_p(X)-r_{p,\max})/\lambda}
    \right]\\
    &\leq
    1-P(A_p)
    =
    1-\frac{1}{C_{\infty,p}}.
\end{align*}
Indeed,
\begin{align*}
\operatorname{Reg}_{n,\lambda}
={}&
B_t(\lambda)
+
\left(
\mathbb E_{P_{\lambda,t}^\star}[r_t(X)]
-
\mathbb E_{P_{\lambda,p}^\star}[r_t(X)]
\right)\\
&+
\left(
\mathbb E_{P_{\lambda,p}^\star}[r_t(X)]
-
\mathbb E_{\widetilde P_{n,\lambda}^{p}}[r_t(X)]
\right).
\end{align*}
Adding the softening-bias, proxy-mismatch, and finite-$n$ terms cancels $D_t$.  
Dropping the nonpositive term $-\lambda D_p$ and applying the last bound on $\rho_{\lambda,p}$ proves the result.
\end{proof}

\subsection{Proofs for Section~\ref{subsec:clipped_exp_gsi}}

Fix a reasoning state $h$.  
Assume that $\pi_B(\cdot\mid h)\ll\pi_S(\cdot\mid h)$ and that $r(h,y)\leq R$ for every relevant $y$.  Define
\begin{equation*}
    Z_{\beta,B}(h)
    :=
    \mathbb E_{\pi_B(\cdot\mid h)}
    \left[
        e^{\beta r(h,Y)}
    \right],
\end{equation*}
and
\begin{equation*}
    \pi_{\beta,B}^\star(y\mid h)
    :=
    \frac{
        \pi_B(y\mid h)e^{\beta r(h,y)}
    }{
        Z_{\beta,B}(h)
    }.
\end{equation*}
Recall that
\begin{equation*}
    s_C(h,y)
    =
    \beta r(h,y)+\min\{d(h,y),C\},
    \qquad
    U_C=\beta R+C.
\end{equation*}

\subsubsection{Proof of Theorem~\ref{thm:clipped_exp_gsi_bias}}

\begin{proof}
Throughout this proof, we suppress the conditioning on the fixed state $h$ when no ambiguity can arise.  Define the clipped tilted target
\begin{equation*}
    \pi_{\beta,C}^\star(y\mid h)
    :=
    \frac{
        \pi_S(y\mid h)e^{s_C(h,y)}
    }{
        Z_{\beta,C}(h)
    },
\end{equation*}
where
\begin{equation*}
    Z_{\beta,C}(h)
    :=
    \mathbb E_{\pi_S(\cdot\mid h)}
    \left[
        e^{s_C(h,Y)}
    \right].
\end{equation*}
For one candidate $Y\sim\pi_S(\cdot\mid h)$ and
$E\sim\operatorname{Exp}(1)$, let
\begin{equation*}
    H
    :=
    \{s_C(h,Y)+E\geq U_C\}.
\end{equation*}
Since $s_C(h,y)\leq U_C$,
\begin{equation*}
    \Pr(Y=y,H\mid h)
    =
    e^{-U_C}
    \pi_S(y\mid h)e^{s_C(h,y)}.
\end{equation*}
Consequently,
\begin{equation*}
    \Pr(Y=y\mid H,h)
    =
    \pi_{\beta,C}^\star(y\mid h),
\end{equation*}
and the per-candidate no-hit probability is
\begin{equation*}
    \rho_{\beta,C}(h)
    =
    1-e^{-U_C}Z_{\beta,C}(h).
\end{equation*}
Conditional on $H$, exponential memorylessness gives
\begin{equation*}
    s_C(h,Y)+E-U_C
    \sim
    \operatorname{Exp}(1),
\end{equation*}
independently of $Y$.  Hence, conditional on the set of hit indices, the hit labels are i.i.d. from $\pi_{\beta,C}^\star(\cdot\mid h)$ and are independent of their exponential overshoots.  Selecting the largest overshoot, as full clipped ExpBoN does, or selecting the first hit in an independent random order, as Algorithm~\ref{alg:clipped_exp_gsi} does, therefore yields the same marginal hit-branch law $\pi_{\beta,C}^\star(\cdot\mid h)$.
More explicitly, conditional on any nonempty hit set, both rules select an index uniformly from that set, independently of the i.i.d. hit labels. If no hit occurs, Algorithm~\ref{alg:clipped_exp_gsi} evaluates all candidates and returns the exponential-noise maximizer formed from the same clipped scores and noises. Thus, the index $i^\star$ selected by Lines~1--14 has exactly the finite-$n$ clipped ExpBoN law; the subsequent GSI acceptance gate and fallback are not part of $\widetilde\pi_{n,\beta,C}$.
Therefore,
\begin{equation*}
    \mathrm{TV}
    \left(
        \widetilde\pi_{n,\beta,C},
        \pi_{\beta,C}^\star
    \right)
    \leq
    \rho_{\beta,C}(h)^n.
\end{equation*}

Define
\begin{equation*}
    w_C(h,y)
    :=
    \min\left\{
        1,
        e^{C-d(h,y)}
    \right\}.
\end{equation*}
Then
\begin{equation*}
    \pi_S(y\mid h)e^{s_C(h,y)}
    =
    \pi_B(y\mid h)e^{\beta r(h,y)}w_C(h,y).
\end{equation*}
Writing
\begin{equation*}
    m_C(h)
    :=
    \mathbb E_{\pi_{\beta,B}^\star(\cdot\mid h)}
    [w_C(h,Y)]
    =
    1-\tau_C(h),
\end{equation*}
we obtain
\begin{equation*}
    \pi_{\beta,C}^\star(y\mid h)
    =
    \frac{
        \pi_{\beta,B}^\star(y\mid h)w_C(h,y)
    }{
        m_C(h)
    }.
\end{equation*}
Since $0\leq w_C\leq1$ and $0<m_C\leq1$,
\begin{align*}
&\mathrm{TV}
\left(
    \pi_{\beta,C}^\star,
    \pi_{\beta,B}^\star
\right)\\
&\quad=
1-
\mathbb E_{\pi_{\beta,B}^\star}
\left[
    \min\left\{
        \frac{w_C}{m_C},
        1
    \right\}
\right]\\
&\quad\leq
1-
\mathbb E_{\pi_{\beta,B}^\star}[w_C]
=
\tau_C(h).
\end{align*}
The triangle inequality proves
\eqref{eq:finite_n_plus_clipping_bias}.  Finally,
\begin{equation*}
    \left|
        \mathbb E_Q[g]-\mathbb E_{Q'}[g]
    \right|
    \leq
    \left(
        \sup_y g(y)-\inf_y g(y)
    \right)
    \mathrm{TV}(Q,Q')
\end{equation*}
gives~\eqref{eq:finite_n_plus_clipping_utility}.
\end{proof}

\subsection{Report-Noisy-Max with Gumbel Noise}
\label{subapp::SBoN_gumbel}

We show that the final selection step of SBoN is equivalent in distribution to report-noisy-max with independent standard Gumbel noises: 
Fix $X_1,\ldots,X_n$ and let
$G_1,\ldots,G_n\overset{\mathrm{iid}}{\sim}\operatorname{Gumbel}(0,1)$,
independently of the candidates.
Define
\begin{equation*}
K
:=
\arg\max_{i=1,\ldots,n}
\left\{
\frac{r(X_i)}{\lambda}+G_i
\right\}.
\end{equation*}
Then, for every $i\in\{1,\ldots,n\}$,
\begin{equation}
\label{eq:sbon_gumbel_probability}
\Pr \left(K=i\mid X_1,\ldots,X_n\right)
=
\frac{
e^{r(X_i)/\lambda}
}{
\sum_{j=1}^n e^{r(X_j)/\lambda}
}.
\end{equation}

\begin{proof}
Condition on $X_1,\ldots,X_n$ and write
\begin{equation*}
a_i:=\frac{r(X_i)}{\lambda},
\qquad
i=1,\ldots,n.
\end{equation*}
The cumulative distribution function and density of a standard Gumbel random
variable are
\begin{equation*}
F_G(g)
=
\exp\left(-e^{-g}\right)
\end{equation*}
and
\begin{equation*}
f_G(g)
=
e^{-g}
\exp\left(-e^{-g}\right),
\end{equation*}
respectively.  Since the Gumbel distribution is continuous, the maximizer is
unique almost surely.  Hence,
\begin{align*}
& \Pr\left(K=i\mid X_1,\ldots,X_n\right)
\\
&\quad=
\int_{-\infty}^{\infty}
f_G(t-a_i)
\prod_{j\neq i}
F_G(t-a_j)
,dt.
\end{align*}
Using the expressions for $f_G$ and $F_G$, the integrand becomes
\begin{align*}
&f_G(t-a_i)
\prod_{j\neq i}
F_G(t-a_j)
\\
& =
e^{-(t-a_i)}
\exp\left(
-e^{-(t-a_i)}
\right)
\prod_{j\neq i}
\exp\left(
-e^{-(t-a_j)}
\right)
\\
& =
e^{a_i}e^{-t}
\exp\left(
-e^{-t}
\sum_{j=1}^n e^{a_j}
\right).
\end{align*}
Let
\begin{equation*}
S:=\sum_{j=1}^n e^{a_j}
\end{equation*}
and make the change of variables
\begin{equation*}
u:=Se^{-t}.
\end{equation*}
Then
\begin{equation*}
du=-Se^{-t}\,dt,
\end{equation*}
and therefore
\begin{align*}
&\Pr\left(K=i\mid X_1,\ldots,X_n\right)
\\ 
& = 
e^{a_i}
\int_{-\infty}^{\infty}
e^{-t}
\exp\left(-Se^{-t}\right) dt
\\
& = 
\frac{e^{a_i}}{S}
\int_0^\infty e^{-u}\,du
\\ 
& =
\frac{e^{a_i}}{\sum_{j=1}^n e^{a_j}}.
\end{align*}
Substituting $a_j=r(X_j)/\lambda$ proves
\eqref{eq:sbon_gumbel_probability}.
\end{proof}

\subsection{More Implementation Details}
\label{sec:implementation_details}

\paragraph{Models and serving.}
The Qwen2.5-Math~\citep{yang2024qwen2} configuration uses
\textsc{Qwen2.5-Math-1.5B-Instruct} as the draft policy $\pi_S$, \textsc{Qwen2.5-Math-7B-Instruct} as the target policy $\pi_B$, and \textsc{Qwen2.5-Math-PRM-7B} as the reward model; the Qwen3~\citep{yang2025qwen3} configuration uses \textsc{Qwen3-1.7B} and \textsc{Qwen3-14B} with the same PRM and thinking mode disabled. All models are served as separate vLLM~\citep{kwon2023efficient} instances with prefix caching enabled. Qwen2.5-Math runs use two NVIDIA A100-40 GPUs per evaluation run (draft and target colocated on one device, PRM on the other); Qwen3 runs host all three models on a single A100-80GB. Step timing excludes one-time engine initialization. Because the two families use different serving layouts, wall-clock numbers should be compared within, not across, model families. Draft- and target-policy log-probabilities are obtained through vLLM's prompt-logprobs interface with prefix caching enabled. Since prefix-cached positions return placeholder entries, we apply a thin runtime patch that marks such rows, validate every returned log-probability (rejecting non-finite or positive values), and transparently re-fetch affected candidates with the cache bypassed; the patch does not alter vLLM's sampling or scheduling behavior and is included in the code release.

\paragraph{Benchmarks and evaluation.}
We evaluate on MATH500~\citep{lightman2024let}, MMLU-STEM~\citep{hendrycks2020measuring}, and Minerva Math~\citep{lewkowycz2022solving}. For MATH500 and MMLU-STEM we use fixed 400-problem subsets (for MATH500, drawn from the pool that remains after removing the 60-problem calibration block); for Minerva Math we use the full 272-problem set. The same problems and random seeds are shared across all methods and candidate budgets, enabling paired comparisons. Results are reported as means and standard deviations over three random seeds for Qwen2.5-Math and two random seeds for Qwen3.

\paragraph{Inference configuration.}
Following GSI~\citep{geuter2025guided}, all methods use candidate budgets $n\in\{2,4,8,16\}$ ($n\in\{4,16\}$ for Qwen3), inverse temperature $\beta=20$, acceptance threshold $u=0.5$, sampling temperature $0.7$, top-$p=1.0$, and at most 512 new tokens per reasoning step. RSD~\citep{liao2025reward} follows the implementation and hyperparameter configuration adopted by GSI, with acceptance threshold $\delta=0.7$. The early-exit scan first evaluates a batch of $b=\max(1,\,\lfloor n/4\rfloor)$ candidates (i.e., $b=1,1,2,4$ for $n=2,4,8,16$); if none passes the exit test, all remaining candidates are evaluated in a single second batch, which for prefix-cache efficiency and reduces scheduling overhead. The batch schedule is a serving-level choice fixed a priori and not tuned on any evaluation set: given the presampled candidate order, the algorithm always selects the first candidate that passes the exit test, so batching changes only the realized computation, never the selected step.

\paragraph{Calibration of the clipping level.}
The only hyperparameter ExpGSI adds over GSI is the clipping level $C$, which acts as a working upper bound on the per-step log-likelihood ratio $d=\log\pi_B(y\mid h)-\log\pi_S(y\mid h)$: the early-exit ceiling $\beta+C$ is valid precisely because $d$ rarely exceeds $C$. We calibrate $C$ as a robust estimate of this upper bound. Rolling out $60$ MATH-500 problems held out from all evaluation subsets, with $64$ draft candidates per reasoning step, we compute $d$ for every candidate and set $C$ to the empirical $95$th percentile, yielding $C=0.45$. A percentile is preferred over the sample maximum because $d$ is heavy-tailed: the maximum would inflate the ceiling--and thus suppress early exits--for the sake of rare outliers, whereas the ${\sim}5\%$ of candidates with $d>C$ are simply clipped--a change to the target distribution rather than a sampling error, bounded in Theorem~\ref{thm:clipped_exp_gsi_bias}--and, in practice, the selection rule is left nearly intact. Calibrated once, $C$ is held fixed across all benchmarks, candidate budgets, and both model families without retuning; accuracy is insensitive to this choice over a wide range (Table~\ref{tab:clip_sensitivity}).

\paragraph{Computational accounting.}
Estimated FLOPs follow the convention of RSD~\citep{liao2025reward}: each token processed by a model costs $2N$ FLOPs, where $N$ is the model's parameter count, and we sum over the forward passes each method actually executes. This comprises (i)~generation, charged to the generating policy, including target-policy fallback regeneration; (ii)~reward-model scoring, where each call re-reads the full sequence (prompt, accepted prefix, and candidate steps), consistent with the official RSD implementation, since no key--value cache persists across PRM calls; and (iii)~the target-policy log-likelihoods used by GSI and ExpGSI, which are served with prefix caching and hence charged only for the candidate suffix, while cache cold starts and refetches are charged at full sequence length. Per-problem totals are averaged over the evaluation set.

\paragraph{Computing infrastructure.}
Qwen2.5-Math experiments (including the reward-model and clipping-level ablations) ran on a server with $8\times$ NVIDIA A100-SXM4-40GB GPUs, AMD EPYC 7542 CPUs (128 hardware threads), 2\,TB RAM, and Oracle Linux 9.7, allocating two GPUs per evaluation run as described above. Qwen3 experiments ran on a cloud instance with $6\times$ NVIDIA A100-SXM4-80GB GPUs and 2\,TB RAM (Linux), hosting all three models of a run on a single GPU. Both machines used the same software stack: Python~3.12, PyTorch~2.11 (CUDA~13), vLLM~0.24, and Transformers~5.13; a complete dependency freeze is included in the code release.

\subsection{Additional experimental results}
\label{sec:Additional experimental results}
% ============================================================
% MATH500
% ============================================================

\begin{table*}[t]
\centering
\small
\setlength{\tabcolsep}{5pt}
\renewcommand{\arraystretch}{0.95}
\begin{tabular}{@{}c l c c c c@{}}
\toprule
$n$ & Method & Acc. (\%) & Time/step (s)
& Accept. (\%) & Est. TFLOPs/prob. \\
\midrule

\multirow{5}{*}{2}
& S-BoN ($\pi_S$) & $76.5 \pm 0.4$ & $0.40 \pm 0.00$ & -- & $121 \pm 5$ \\
& RSD$^\dagger$ & $79.3 \pm 1.9$ & $0.49 \pm 0.02$ & $96.7 \pm 0.3$ & $118 \pm 3$ \\
& S-BoN ($\pi_B$) & $83.2 \pm 1.1$ & $0.91 \pm 0.00$ & -- & $148 \pm 1$ \\
& GSI & $80.3 \pm 1.2$ & $0.66 \pm 0.02$ & $89.6 \pm 0.5$ & $174 \pm 5$ \\
& \textbf{ExpGSI (ours)} & $80.2 \pm 0.7$ & $0.61 \pm 0.02$
& $89.5 \pm 0.8$ & $\mathbf{147 \pm 3}$ \\
\midrule

\multirow{5}{*}{4}
& S-BoN ($\pi_S$) & $78.1 \pm 0.8$ & $0.52 \pm 0.03$ & -- & $244 \pm 22$ \\
& RSD$^\dagger$ & $80.0 \pm 1.1$ & $0.59 \pm 0.02$ & $97.6 \pm 0.1$ & $247 \pm 5$ \\
& S-BoN ($\pi_B$) & $83.2 \pm 0.8$ & $1.11 \pm 0.02$ & -- & $322 \pm 6$ \\
& GSI & $81.8 \pm 1.4$ & $0.78 \pm 0.02$ & $92.4 \pm 0.4$ & $381 \pm 8$ \\
& \textbf{ExpGSI (ours)} & $82.3 \pm 1.7$ & $0.74 \pm 0.02$
& $92.3 \pm 0.7$ & $\mathbf{288 \pm 5}$ \\
\midrule

\multirow{5}{*}{8}
& S-BoN ($\pi_S$) & $80.0 \pm 0.7$ & $0.71 \pm 0.03$ & -- & $532 \pm 40$ \\
& RSD$^\dagger$ & $80.6 \pm 1.2$ & $0.77 \pm 0.01$ & $98.1 \pm 0.6$ & $511 \pm 29$ \\
& S-BoN ($\pi_B$) & $84.5 \pm 0.5$ & $1.44 \pm 0.03$ & -- & $647 \pm 30$ \\
& GSI & $82.3 \pm 1.2$ & $1.01 \pm 0.01$ & $94.1 \pm 0.2$ & $835 \pm 14$ \\
& \textbf{ExpGSI (ours)} & $82.1 \pm 0.6$ & $\mathbf{0.91 \pm 0.01}$
& $94.3 \pm 0.1$ & $\mathbf{524 \pm 22}$ \\
\midrule

\multirow{5}{*}{16}
& S-BoN ($\pi_S$) & $80.4 \pm 0.4$ & $1.08 \pm 0.02$ & -- & $1130 \pm 14$ \\
& RSD$^\dagger$ & $80.2 \pm 0.9$ & $1.15 \pm 0.05$ & $98.3 \pm 0.4$ & $1046 \pm 17$ \\
& S-BoN ($\pi_B$) & $84.2 \pm 1.4$ & $1.96 \pm 0.02$ & -- & $1366 \pm 16$ \\
& GSI & $82.2 \pm 1.4$ & $1.52 \pm 0.06$ & $95.4 \pm 0.5$ & $1679 \pm 10$ \\
& \textbf{ExpGSI (ours)} & $81.8 \pm 0.9$ & $\mathbf{1.17 \pm 0.02}$
& $95.6 \pm 0.3$ & $\mathbf{900 \pm 24}$ \\
\bottomrule
\end{tabular}

\caption{
Complete results on MATH500 across candidate budgets.
Reported values are means and standard deviations over three random seeds.
Bold efficiency values indicate improvements of ExpGSI over GSI.
$^\dagger$RSD uses the implementation and hyperparameter configuration
adopted by GSI.
}
\label{tab:app_math500_complete}
\end{table*}

% ============================================================
% MMLU-STEM
% ============================================================

\begin{table*}[t]
\centering
\small
\setlength{\tabcolsep}{5pt}
\renewcommand{\arraystretch}{0.95}
\begin{tabular}{@{}c l c c c c@{}}
\toprule
$n$ & Method & Acc. (\%) & Time/step (s)
& Accept. (\%) & Est. TFLOPs/prob. \\
\midrule

\multirow{5}{*}{2}
& S-BoN ($\pi_S$) & $59.2 \pm 2.1$ & $0.35 \pm 0.00$ & -- & $107 \pm 0$ \\
& RSD$^\dagger$ & $67.1 \pm 0.8$ & $0.48 \pm 0.01$ & $89.4 \pm 0.4$ & $92 \pm 2$ \\
& S-BoN ($\pi_B$) & $68.8 \pm 2.3$ & $0.79 \pm 0.01$ & -- & $100 \pm 1$ \\
& GSI & $68.4 \pm 1.5$ & $0.81 \pm 0.01$ & $62.4 \pm 1.0$ & $133 \pm 0$ \\
& \textbf{ExpGSI (ours)} & $67.2 \pm 1.2$ & $\mathbf{0.77 \pm 0.01}$
& $62.6 \pm 1.0$ & $\mathbf{110 \pm 3}$ \\
\midrule

\multirow{5}{*}{4}
& S-BoN ($\pi_S$) & $64.0 \pm 0.7$ & $0.44 \pm 0.01$ & -- & $212 \pm 9$ \\
& RSD$^\dagger$ & $67.4 \pm 2.0$ & $0.56 \pm 0.01$ & $92.9 \pm 0.3$ & $199 \pm 7$ \\
& S-BoN ($\pi_B$) & $72.0 \pm 2.0$ & $0.94 \pm 0.01$ & -- & $207 \pm 3$ \\
& GSI & $69.0 \pm 0.4$ & $0.94 \pm 0.02$ & $69.9 \pm 0.3$ & $283 \pm 4$ \\
& \textbf{ExpGSI (ours)} & $67.8 \pm 1.8$ & $\mathbf{0.90 \pm 0.02}$
& $69.6 \pm 0.9$ & $\mathbf{224 \pm 5}$ \\
\midrule

\multirow{5}{*}{8}
& S-BoN ($\pi_S$) & $63.6 \pm 1.7$ & $0.61 \pm 0.01$ & -- & $517 \pm 54$ \\
& RSD$^\dagger$ & $68.8 \pm 2.1$ & $0.76 \pm 0.01$ & $94.6 \pm 0.1$ & $468 \pm 21$ \\
& S-BoN ($\pi_B$) & $73.1 \pm 1.0$ & $1.18 \pm 0.02$ & -- & $423 \pm 12$ \\
& GSI & $68.9 \pm 1.3$ & $1.23 \pm 0.04$ & $75.5 \pm 0.6$ & $581 \pm 2$ \\
& \textbf{ExpGSI (ours)} & $69.3 \pm 3.0$ & $\mathbf{1.15 \pm 0.00}$
& $75.1 \pm 0.4$ & $\mathbf{425 \pm 5}$ \\
\midrule

\multirow{5}{*}{16}
& S-BoN ($\pi_S$) & $67.5 \pm 1.7$ & $0.94 \pm 0.02$ & -- & $1072 \pm 54$ \\
& RSD$^\dagger$ & $69.2 \pm 1.7$ & $1.10 \pm 0.03$ & $95.7 \pm 0.3$ & $1020 \pm 84$ \\
& S-BoN ($\pi_B$) & $71.6 \pm 2.3$ & $1.56 \pm 0.04$ & -- & $869 \pm 8$ \\
& GSI & $69.6 \pm 1.9$ & $1.81 \pm 0.04$ & $80.2 \pm 0.7$ & $1224 \pm 21$ \\
& \textbf{ExpGSI (ours)} & $67.8 \pm 0.4$ & $\mathbf{1.53 \pm 0.04}$
& $79.9 \pm 1.0$ & $\mathbf{822 \pm 20}$ \\
\bottomrule
\end{tabular}

\caption{
Complete results on MMLU-STEM across candidate budgets.
Reported values are means and standard deviations over three random seeds.
Bold efficiency values indicate improvements of ExpGSI over GSI.
$^\dagger$RSD uses the implementation and hyperparameter configuration
adopted by GSI.
}
\label{tab:app_mmlu_stem_complete}
\end{table*}

% ============================================================
% Minerva Math
% ============================================================

\begin{table*}[t]
\centering
\small
\setlength{\tabcolsep}{5pt}
\renewcommand{\arraystretch}{0.95}
\begin{tabular}{@{}c l c c c c@{}}
\toprule
$n$ & Method & Acc. (\%) & Time/step (s)
& Accept. (\%) & Est. TFLOPs/prob. \\
\midrule

\multirow{5}{*}{2}
& S-BoN ($\pi_S$) & $25.0 \pm 1.3$ & $0.39 \pm 0.01$ & -- & $149 \pm 5$ \\
& RSD$^\dagger$ & $24.9 \pm 1.7$ & $0.50 \pm 0.02$ & $94.1 \pm 0.2$ & $150 \pm 4$ \\
& S-BoN ($\pi_B$) & $31.9 \pm 1.5$ & $0.90 \pm 0.02$ & -- & $171 \pm 5$ \\
& GSI & $28.6 \pm 1.1$ & $0.77 \pm 0.01$ & $78.1 \pm 0.3$ & $198 \pm 6$ \\
& \textbf{ExpGSI (ours)} & $29.7 \pm 0.6$ & $\mathbf{0.73 \pm 0.02}$
& $78.5 \pm 0.8$ & $\mathbf{175 \pm 5}$ \\
\midrule

\multirow{5}{*}{4}
& S-BoN ($\pi_S$) & $25.7 \pm 0.4$ & $0.51 \pm 0.01$ & -- & $301 \pm 10$ \\
& RSD$^\dagger$ & $24.4 \pm 1.1$ & $0.60 \pm 0.02$ & $96.6 \pm 0.5$ & $325 \pm 11$ \\
& S-BoN ($\pi_B$) & $31.2 \pm 0.6$ & $1.13 \pm 0.01$ & -- & $363 \pm 10$ \\
& GSI & $28.9 \pm 0.8$ & $0.90 \pm 0.01$ & $84.1 \pm 0.0$ & $445 \pm 11$ \\
& \textbf{ExpGSI (ours)} & $29.4 \pm 1.1$ & $\mathbf{0.88 \pm 0.02}$
& $83.7 \pm 0.5$ & $\mathbf{342 \pm 10}$ \\
\midrule

\multirow{5}{*}{8}
& S-BoN ($\pi_S$) & $25.2 \pm 0.8$ & $0.74 \pm 0.00$ & -- & $650 \pm 10$ \\
& RSD$^\dagger$ & $24.9 \pm 1.3$ & $0.81 \pm 0.01$ & $97.3 \pm 0.2$ & $654 \pm 17$ \\
& S-BoN ($\pi_B$) & $32.5 \pm 2.1$ & $1.44 \pm 0.05$ & -- & $758 \pm 11$ \\
& GSI & $28.8 \pm 1.5$ & $1.18 \pm 0.02$ & $88.2 \pm 0.4$ & $949 \pm 38$ \\
& \textbf{ExpGSI (ours)} & $29.0 \pm 0.7$ & $\mathbf{1.11 \pm 0.05}$
& $87.6 \pm 0.9$ & $\mathbf{693 \pm 26}$ \\
\midrule

\multirow{5}{*}{16}
& S-BoN ($\pi_S$) & $26.5 \pm 1.0$ & $1.14 \pm 0.02$ & -- & $1343 \pm 19$ \\
& RSD$^\dagger$ & $26.2 \pm 2.4$ & $1.27 \pm 0.02$ & $98.0 \pm 0.4$ & $1427 \pm 42$ \\
& S-BoN ($\pi_B$) & $31.6 \pm 1.0$ & $1.99 \pm 0.12$ & -- & $1538 \pm 21$ \\
& GSI & $28.7 \pm 0.7$ & $1.80 \pm 0.02$ & $90.6 \pm 0.2$ & $1977 \pm 39$ \\
& \textbf{ExpGSI (ours)} & $29.4 \pm 1.9$ & $\mathbf{1.50 \pm 0.02}$
& $90.5 \pm 0.4$ & $\mathbf{1274 \pm 50}$ \\
\bottomrule
\end{tabular}

\caption{
Complete results on Minerva Math across candidate budgets. Reported values are means and standard deviations over three random seeds. Bold efficiency values indicate improvements of ExpGSI over GSI. $^\dagger$RSD uses the implementation and hyperparameter configuration adopted by GSI.
}
\label{tab:app_minerva_complete}
\end{table*}

\paragraph{Results across candidate budgets (Qwen2.5-Math).}
Tables~\ref{tab:app_math500_complete}--\ref{tab:app_minerva_complete} report the complete per-benchmark results. The computational advantage of ExpGSI over GSI increases with the candidate budget on every benchmark, reaching $33$--$46\%$ of estimated compute at $n{=}16$ with latency reductions of $16$--$23\%$. Acceptance rates remain close to GSI's throughout, and accuracy differences show no consistent direction. ExpGSI thus reduces the selection overhead of GSI without systematically changing its solution quality.

\begin{table*}[t]
\centering
\small
\setlength{\tabcolsep}{5pt}
\renewcommand{\arraystretch}{0.95}
\begin{tabular}{@{}c c l c c c c@{}}
\toprule
Benchmark & $n$ & Method & Acc. (\%) & Time/step (s)
& Accept. (\%) & Est. TFLOPs/prob. \\
\midrule

\multirow{10}{*}{MATH500}
& \multirow{5}{*}{4}
& S-BoN ($\pi_S$) & $67.6 \pm 0.2$ & $0.27 \pm 0.00$ & -- & $714 \pm 13$ \\
& & RSD$^\dagger$ & $67.2 \pm 0.0$ & $0.28 \pm 0.00$ & $99.5 \pm 0.0$ & $715 \pm 5$ \\
& & S-BoN ($\pi_B$) & $74.0 \pm 0.4$ & $0.75 \pm 0.00$ & -- & $802 \pm 1$ \\
& & GSI & $68.2 \pm 0.7$ & $0.46 \pm 0.01$ & $98.4 \pm 0.2$ & $1455 \pm 1$ \\
& & \textbf{ExpGSI (ours)} & $67.9 \pm 0.2$ & $\mathbf{0.40 \pm 0.01}$
& $98.2 \pm 0.1$ & $\mathbf{1006 \pm 7}$ \\
\cmidrule(lr){2-7}
& \multirow{5}{*}{16}
& S-BoN ($\pi_S$) & $68.4 \pm 0.5$ & $0.70 \pm 0.00$ & -- & $2874 \pm 14$ \\
& & RSD$^\dagger$ & $69.0 \pm 1.4$ & $0.71 \pm 0.00$ & $99.7 \pm 0.0$ & $2884 \pm 50$ \\
& & S-BoN ($\pi_B$) & $73.9 \pm 0.2$ & $1.33 \pm 0.00$ & -- & $3167 \pm 71$ \\
& & GSI & $68.5 \pm 0.4$ & $1.34 \pm 0.01$ & $99.2 \pm 0.2$ & $6644 \pm 37$ \\
& & \textbf{ExpGSI (ours)} & $69.0 \pm 1.1$ & $\mathbf{0.80 \pm 0.01}$
& $99.3 \pm 0.0$ & $\mathbf{3181 \pm 11}$ \\
\midrule

\multirow{10}{*}{MMLU-STEM}
& \multirow{5}{*}{4}
& S-BoN ($\pi_S$) & $69.8 \pm 0.7$ & $0.28 \pm 0.01$ & -- & $374 \pm 4$ \\
& & RSD$^\dagger$ & $71.5 \pm 1.1$ & $0.33 \pm 0.00$ & $96.5 \pm 0.0$ & $395 \pm 10$ \\
& & S-BoN ($\pi_B$) & $86.0 \pm 0.4$ & $0.80 \pm 0.00$ & -- & $510 \pm 7$ \\
& & GSI & $82.1 \pm 1.2$ & $0.70 \pm 0.02$ & $80.5 \pm 0.3$ & $925 \pm 1$ \\
& & \textbf{ExpGSI (ours)} & $80.9 \pm 0.5$ & $\mathbf{0.67 \pm 0.00}$
& $80.3 \pm 0.4$ & $\mathbf{736 \pm 14}$ \\
\cmidrule(lr){2-7}
& \multirow{5}{*}{16}
& S-BoN ($\pi_S$) & $69.8 \pm 0.4$ & $0.67 \pm 0.00$ & -- & $1616 \pm 41$ \\
& & RSD$^\dagger$ & $74.2 \pm 1.1$ & $0.72 \pm 0.01$ & $97.7 \pm 0.1$ & $1629 \pm 22$ \\
& & S-BoN ($\pi_B$) & $86.0 \pm 1.8$ & $1.30 \pm 0.01$ & -- & $2002 \pm 12$ \\
& & GSI & $81.5 \pm 0.0$ & $1.53 \pm 0.00$ & $85.9 \pm 0.0$ & $4558 \pm 87$ \\
& & \textbf{ExpGSI (ours)} & $81.2 \pm 0.4$ & $\mathbf{1.18 \pm 0.03}$
& $86.0 \pm 0.4$ & $\mathbf{2831 \pm 41}$ \\
\midrule

\multirow{10}{*}{Minerva Math}
& \multirow{5}{*}{4}
& S-BoN ($\pi_S$) & $26.1 \pm 2.1$ & $0.30 \pm 0.00$ & -- & $475 \pm 6$ \\
& & RSD$^\dagger$ & $28.7 \pm 0.5$ & $0.34 \pm 0.00$ & $97.1 \pm 0.2$ & $513 \pm 7$ \\
& & S-BoN ($\pi_B$) & $37.9 \pm 1.6$ & $0.78 \pm 0.01$ & -- & $613 \pm 14$ \\
& & GSI & $33.3 \pm 0.3$ & $0.63 \pm 0.00$ & $89.3 \pm 0.2$ & $1226 \pm 26$ \\
& & \textbf{ExpGSI (ours)} & $32.5 \pm 0.8$ & $\mathbf{0.57 \pm 0.00}$
& $89.9 \pm 0.5$ & $\mathbf{932 \pm 13}$ \\
\cmidrule(lr){2-7}
& \multirow{5}{*}{16}
& S-BoN ($\pi_S$) & $26.3 \pm 1.3$ & $0.75 \pm 0.00$ & -- & $2058 \pm 65$ \\
& & RSD$^\dagger$ & $27.4 \pm 1.8$ & $0.78 \pm 0.00$ & $98.0 \pm 0.0$ & $2095 \pm 29$ \\
& & S-BoN ($\pi_B$) & $37.9 \pm 1.6$ & $1.39 \pm 0.02$ & -- & $2527 \pm 18$ \\
& & GSI & $32.9 \pm 0.8$ & $1.65 \pm 0.00$ & $93.0 \pm 0.1$ & $6106 \pm 5$ \\
& & \textbf{ExpGSI (ours)} & $35.3 \pm 3.1$ & $\mathbf{1.16 \pm 0.00}$
& $93.7 \pm 0.4$ & $\mathbf{3495 \pm 48}$ \\
\bottomrule
\end{tabular}

\caption{
Complete per-benchmark results for the Qwen3 configuration (\textsc{Qwen3-1.7B} draft, \textsc{Qwen3-14B} target, thinking mode disabled, same PRM and clipping level as the main experiments) at $n\in\{4,16\}$. Reported values are means and standard deviations over two random seeds. Bold efficiency values indicate improvements of ExpGSI over GSI. $^\dagger$RSD uses the implementation and hyperparameter configuration adopted by GSI.
}
\label{tab:app_qwen3_complete}
\end{table*}

\paragraph{Per-benchmark results for Qwen3.}
Table~\ref{tab:app_qwen3_complete} breaks the Qwen3 rows of Table~\ref{tab:model_family_average} down by benchmark. The picture from the Qwen2.5-Math experiments transfers: accuracy and acceptance remain at GSI's level on every benchmark (the largest deviation, on Minerva Math, lies within its seed spread), while at $n{=}16$ estimated computation drops by $38$--$52\%$ and time per step by $23$--$40\%$. The larger savings relative to Qwen2.5-Math are consistent with the larger target model: the avoided target-policy evaluations account for a larger share of total cost.

\paragraph{Reward-model ablation (Skywork PRM).}
Table~\ref{tab:skywork_ablation} reports the complete results of the reward-model ablation described in Section~\ref{sec:analysis}: all five methods on MATH500 at $n\in\{4,16\}$ over three random seeds, with the 7B math-specialized PRM replaced by \textsc{Skywork-o1-Open-PRM-Qwen-2.5-1.5B}~\citep{team2024skywork} and all other settings kept identical to the main experiments. The table follows the same format as the per-benchmark tables above, reporting accuracy, time per step, acceptance rate, and estimated TFLOPs per problem.

\begin{table}[t]
\centering
\scriptsize
\setlength{\tabcolsep}{1.6pt}
\renewcommand{\arraystretch}{0.94}
\begin{tabular}{@{}c l c c c c@{}}
\toprule
$n$ & Method
& Acc. (\%)
& Time (s)
& Accpt. (\%)
& TFLOPs \\
\midrule

\multirow{5}{*}{4}
& S-BoN ($\pi_S$)
& $76.5{\pm}0.0$
& $.44{\pm}.00$
& --
& $41{\pm}1$ \\

& RSD$^\dagger$
& $79.2{\pm}1.4$
& $.93{\pm}.02$
& $65.5{\pm}.8$
& $43{\pm}2$ \\

& S-BoN ($\pi_B$)
& $82.7{\pm}0.9$
& $1.09{\pm}.02$
& --
& $76{\pm}0$ \\

& GSI
& $83.6{\pm}1.0$
& $.98{\pm}.01$
& $73.0{\pm}1.2$
& $141{\pm}3$ \\

& ExpGSI (ours)
& $83.8{\pm}0.5$
& $1.01{\pm}.03$
& $72.2{\pm}.7$
& $117{\pm}2$ \\
\midrule

\multirow{5}{*}{16}
& S-BoN ($\pi_S$)
& $75.1{\pm}1.7$
& $.67{\pm}.03$
& --
& $148{\pm}6$ \\

& RSD$^\dagger$
& $80.4{\pm}1.0$
& $1.24{\pm}.02$
& $72.7{\pm}.0$
& $147{\pm}4$ \\

& S-BoN ($\pi_B$)
& $84.4{\pm}1.3$
& $1.51{\pm}.03$
& --
& $269{\pm}4$ \\

& GSI
& $83.2{\pm}1.2$
& $1.45{\pm}.07$
& $78.9{\pm}1.2$
& $540{\pm}12$ \\

& ExpGSI (ours)
& $83.1{\pm}0.9$
& $\mathbf{1.41{\pm}.01}$
& $79.5{\pm}2.3$
& $\mathbf{410{\pm}22}$ \\
\bottomrule
\end{tabular}

\caption{Results with the Skywork reward model on MATH500.  Values are means and standard deviations over three seeds.}
\label{tab:skywork_ablation}
\end{table}

\paragraph{Sensitivity to the clipping level.}

\begin{table}[!t]
\centering
\small
\setlength{\tabcolsep}{6pt}
\begin{tabular}{@{}cccc@{}}
\toprule
$C/\widehat C$
& Acc. (\%)
& Scored cand./step
& Est. TFLOPs/prob. \\
\midrule
0.5                  & $82.4 \pm 1.9$ & 8.9  & 892  \\
0.75                 & $83.0 \pm 0.4$ & 8.8  & 892  \\
1.0 (default)        & $82.9 \pm 0.9$ & 9.2  & 941  \\
1.5                  & $82.3 \pm 0.7$ & 9.6  & 951  \\
2.0                  & $80.8 \pm 1.4$ & 10.2 & 993  \\
$\infty$ (no clipping)
                     & $81.9 \pm 0.5$ & 16.0 & 1592 \\
\bottomrule
\end{tabular}
\caption{
Sensitivity of ExpGSI to the clipping level on MATH500 with $n=16$.
Accuracy is reported as mean $\pm$ standard deviation over two seeds.
When $C=\infty$, clipping and early exit are disabled. 
}
\label{tab:clip_sensitivity}
\end{table}

We evaluate ExpGSI on MATH500 with $n=16$ while varying
\[
C/\widehat C\in\{0.5,0.75,1.0,1.5,2.0,\infty\},
\]
where $\widehat C=0.45$ is used in the main experiments. Each configuration is evaluated on the same 400 problems using two random seeds. We report final-answer accuracy, the average number of candidates scored per reasoning step, and estimated computation per problem.  When $C=\infty$, clipping and certificate-based early exit are disabled. The resulting selection rule replaces the Gumbel noise in GSI with exponential noise and scores all $n$ candidates. Accuracy remains broadly stable over the tested clipping levels. Increasing $C$ generally requires scoring more candidates. Removing clipping requires evaluating all 16 candidates and increases TFLOPs per problem. The finite clipping configurations therefore provide substantial computational savings without showing a consistent loss in accuracy.

\paragraph{ExpBoN vs. SBoN on real candidate pools.}
Figure~\ref{fig:realconv} repeats the finite-$n$ comparison of Figure~1 on real candidate pools. Rolling out the calibration problems step by step (Appendix \ref{sec:implementation_details}) yields one pool per reasoning step: the $64$ next-step candidates drawn from $\pi_S$ at that prefix, together with their PRM rewards and log-likelihood ratios. For each pool we take the empirical distribution over its candidates and compute the \emph{exact} finite-$n$ output laws of ExpBoN and SBoN by one-dimensional numerical integration--no sampling is involved, and the integrator reproduces the exact values of the synthetic example in Figure~1.

The two mechanisms are compared on identical scores in four configurations, one per column of Figure~\ref{fig:realconv}. The first two columns use the plain reward score $r/\lambda$ of Section~2: at the deployed temperature $\lambda=1/\beta=1/20$, and at $\lambda=1/2$, the temperature of the synthetic example in Figure~1--so the second column is the direct real-data counterpart of that example. The last two columns use the GSI score of Section~4 without and with clipping ($\beta r+d$ versus $s_C=\beta r+\min(d,C)$; $\beta=20$, $C=0.45$), the latter being exactly the quantity on which ExpGSI and GSI select candidates. In all four configurations ExpBoN converges geometrically--within a handful of candidates for the reward scores, and by $n\approx64$ under the GSI scores--whereas SBoN decays polynomially throughout, matching Theorem~\ref{thm:pf_sbon_kl_tv} and Corollary~\ref{cor:ExpBoN_relative_reward} on real data. Two further observations. First, the likelihood term slows finite-$n$ convergence for both mechanisms--$d$ spreads the scores within a pool, leaving less probability mass near the pool maximum--and clipping partially restores the speed by capping the right tail. Second, slower convergence does not conflict with the compute savings of ExpGSI: convergence is governed by the score dispersion within a pool, whereas early exit is triggered by proximity to the fixed ceiling $U_C$, and the exit test is exact at every finite $n$.

\begin{figure*}[t]
\centering
\includegraphics[width=\textwidth]{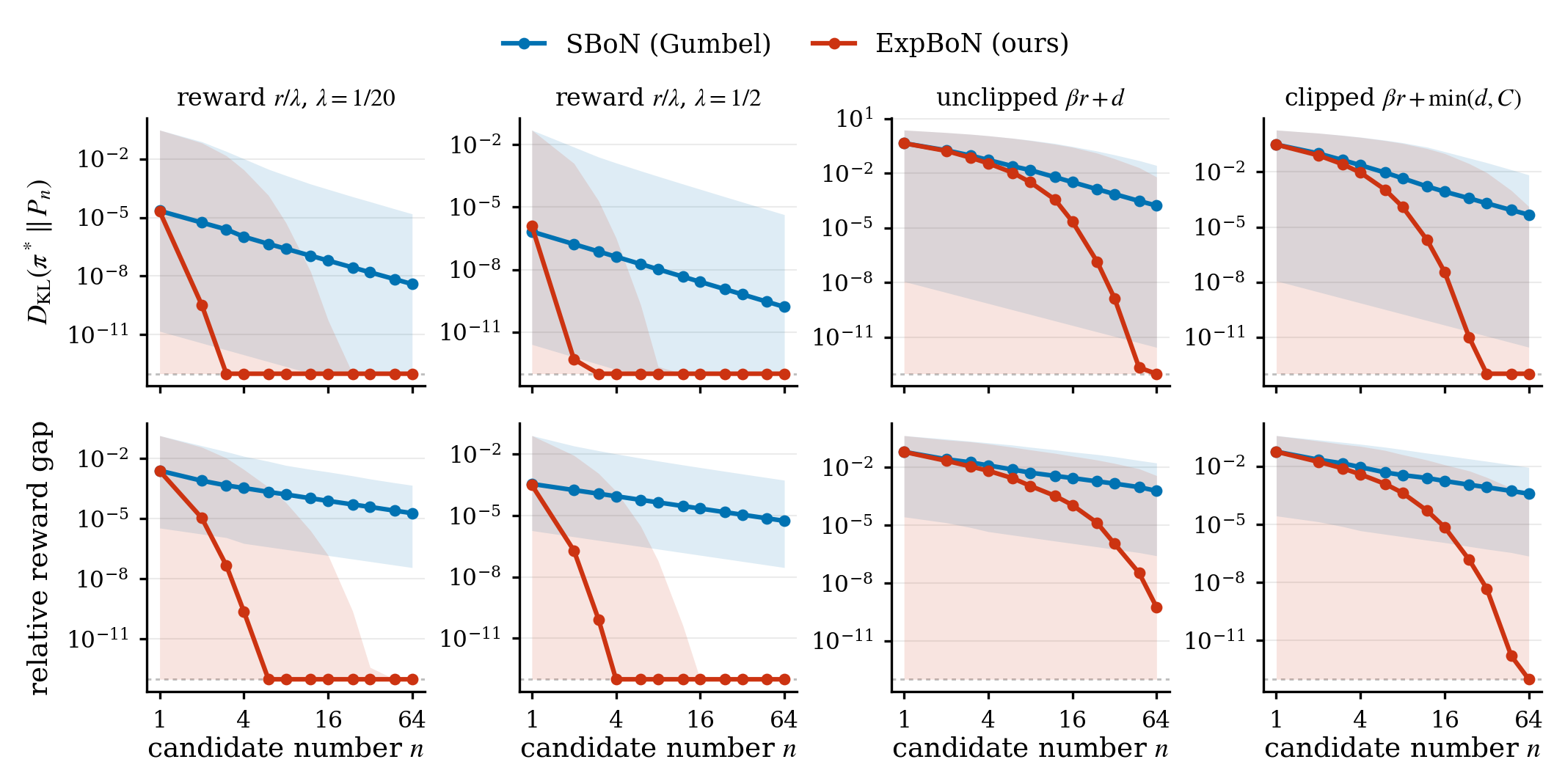}
\caption{
Exact finite-$n$ convergence of ExpBoN vs. SBoN on real LLM candidate pools (medians with 10--90\% bands). Columns: the reward score $r/\lambda$ of Section~2 at the deployed temperature $\lambda=1/20$ and at $\lambda=1/2$, the temperature of the synthetic example in Figure~1; the GSI score of Section~4 without clipping ($\beta r+d$) and with clipping ($s_C=\beta r+\min(d,C)$; $\beta=20$, $C=0.45$). Rows: KL divergence to the tilted target and relative centered-reward gap. Output laws are computed exactly by numerical integration; the dotted line marks numerical precision ($10^{-13}$). ExpBoN converges geometrically in every configuration while SBoN decays polynomially; comparing the last two columns, clipping speeds up convergence by capping the right tail of the scores.
}
\label{fig:realconv}
\end{figure*}

% \clearpage 
% \bibliography{ref}

\end{document}